\documentclass{article}

\usepackage{natbib}
\usepackage[margin=0.97in]{geometry}

\usepackage{times}
\usepackage[utf8]{inputenc} 
\usepackage[T1]{fontenc}    
\usepackage[colorlinks=true, linkcolor=blue, citecolor=blue,urlcolor=black]{hyperref} 
\usepackage{url}            
\usepackage{booktabs}       
\usepackage{amsfonts}       
\usepackage{nicefrac}       
\usepackage{microtype}      
\usepackage{xcolor}         
\usepackage{amsmath}
\usepackage{enumitem}
\usepackage{algorithmicx}

\definecolor{Green}{rgb}{0.13, 0.65, 0.3}
\definecolor{Amber}{rgb}{0.3, 0.5, 1.0}
\usepackage[]{color-edits}
\addauthor{HL}{Green}

\allowdisplaybreaks

\newif\ifnobug
\nobugtrue

\newif\ifspacehack

\usepackage{graphicx}
\usepackage{mathtools}
\usepackage{footnote}
\usepackage{float}
\usepackage{xspace}
\usepackage{multirow}
\usepackage{xcolor}
\usepackage{wrapfig}
\usepackage{framed}
\usepackage{bbm}
\usepackage{footnote}
\usepackage{nicefrac}
\usepackage{makecell}
\usepackage[algo2e, vlined]{algorithm2e}
\usepackage{algorithm}
\usepackage{amssymb}
\usepackage{amsthm}
\usepackage{bm}
\makesavenoteenv{tabular}
\makesavenoteenv{table}

\newtheorem{theorem}{Theorem}[section]

\newtheorem{lemma}[theorem]{Lemma}

\newtheorem{definition}{Definition}

\definecolor{Green}{rgb}{0.13, 0.65, 0.3}
\usepackage[]{color-edits}

\addauthor{HL}{Green}
\addauthor{MZ}{cyan}
\addauthor{CL}{orange}

\newcommand{\calF}{{\mathcal{F}}}

\newcommand{\calK}{{\mathcal{K}}}

\newcommand{\calT}{{\mathcal{T}}}
\newcommand{\calP}{{\mathcal{P}}}

\newcommand{\calN}{{\mathcal{N}}}
\newcommand{\calV}{\mathcal{V}}
\newcommand{\e}{\mathbf{e}}

\newcommand{\Reg}{\text{\rm Reg}}

\newcommand{\one}{\boldsymbol{1}}

\newcommand{\bias}{\textsc{Bias-1}\xspace}
\newcommand{\biastwo}{\textsc{Bias-2}\xspace}
\newcommand{\biasthree}{\textsc{Bias-3}\xspace}
\newcommand{\biasfour}{\textsc{Bias-4}\xspace}

\newcommand{\regterm}{\textsc{Reg-Term}\xspace}

\newcommand{\ramp}{\text{ramp}}

\newcommand{\Pht}{\widehat{P}}
\newcommand{\Pbar}{\overline{P}}
\newcommand{\bonus}{b}

\newcommand{\bonusQ}{B}

\newcommand{\Bonus}{\textsc{Bonus}}
\newcommand{\bonusVec}{\widehat{\Lambda}}
\newcommand{\GR}{\textsc{GeometricResampling}\xspace}

\newcommand{\mix}{\text{mix}}

\newcommand{\explore}{\delta_e}
\newcommand{\cov}{\Sigma}
\newcommand{\hatcov}{\widehat{\Sigma}^+}
\newcommand{\hatcovk}{\widehat{\Sigma}^{+(m)}}

\newcommand{\hattheta}{\widehat{\theta}}
\newcommand{\hatbonusQ}{\widehat{\bonusQ}}
\newcommand{\barell}{\overline{\ell}}
\newcommand{\barQ}{\overline{Q}}
\newcommand{\bartheta}{\overline{\theta}}
\newcommand{\piexp}{\pi_{0}}

\DeclareMathOperator*{\argmin}{argmin}
\DeclareMathOperator*{\argmax}{argmax}

\newcommand{\term}{\textbf{term}}
\newcommand{\R}{R}
\newcommand{\negative}{\textsc{term}_1}
\newcommand{\positive}{\textsc{term}_2}
\newcommand{\loss}{Q}
\newcommand{\bns}{B}
\newcommand{\gap}{\Delta}
\newcommand{\conf}{\textit{conf}}
\newcommand{\de}{\mathrm{d}}

\newcommand{\field}[1]{\mathbb{#1}}

\newcommand{\fR}{\field{R}}

\newcommand{\E}{\field{E}}
\newcommand{\Var}{\text{Var}}

\newcommand{\compuob}{\textsc{Comp-UOB}\xspace}
\newcommand{\complob}{\textsc{Comp-LOB}\xspace}

\newcommand{\inner}[1]{ \left\langle {#1} \right\rangle }
\newcommand{\inn}[1]{ \langle {#1} \rangle }

\newcommand{\norm}[1]{\left\|{#1}\right\|}

\newcommand{\pcov}{\pi_{\text{cov}}}
\newcommand{\findpcov}{\textsc{PolicyCover}\xspace}
\newcommand{\textcov}{\text{cov}}

\newcommand{\wh}{\widehat}

\newcommand{\upconf}{\overline{\QQ}}
\newcommand{\lowconf}{\underline{\QQ}}
\newcommand{\ind}{\mathbbm{1}}

\newcommand{\Qht}{\wh{Q}}
\newcommand{\Vht}{\wh{V}}
\newcommand{\Lyr}{H}
\newcommand{\Q}{{Q}}
\newcommand{\QQ}{{q}}

\newcommand{\qstar}{{q^\star}}

\newcommand{\pistar}{{\pi^{\star}}}

\newcommand{\Ut}{{\upconf_t}}
\newcommand{\Lt}{{\lowconf_t}}

\newtheorem{assumption}{Assumption}

\newcommand{\order}{\ensuremath{\mathcal{O}}}
\newcommand{\otil}{\ensuremath{\widetilde{\mathcal{O}}}}

\newcommand{\oracle}{\textsc{Simulator}}
\newcommand{\xtilde}{\tilde{x}}
\newcommand{\Phat}{\widehat{P}}
\newcommand{\obj}{\textsc{Optimize}\xspace}

\newcommand{\len}{W}

\newcommand{\err}{\varepsilon}

\usepackage{comment}
\usepackage{prettyref}
\newcommand{\pref}[1]{\prettyref{#1}}

\newcommand{\savehyperref}[2]{\texorpdfstring{\hyperref[#1]{#2}}{#2}}
\newrefformat{eq}{\savehyperref{#1}{Eq. \textup{(\ref*{#1})}}}
\newrefformat{eqn}{\savehyperref{#1}{Equation~\ref*{#1}}}
\newrefformat{lem}{\savehyperref{#1}{Lemma~\ref*{#1}}}
\newrefformat{def}{\savehyperref{#1}{Definition~\ref*{#1}}}
\newrefformat{line}{\savehyperref{#1}{Line~\ref*{#1}}}
\newrefformat{thm}{\savehyperref{#1}{Theorem~\ref*{#1}}}
\newrefformat{corr}{\savehyperref{#1}{Corollary~\ref*{#1}}}
\newrefformat{cor}{\savehyperref{#1}{Corollary~\ref*{#1}}}
\newrefformat{sec}{\savehyperref{#1}{Section~\ref*{#1}}}
\newrefformat{app}{\savehyperref{#1}{Appendix~\ref*{#1}}}
\newrefformat{assum}{\savehyperref{#1}{Assumption~\ref*{#1}}}
\newrefformat{ex}{\savehyperref{#1}{Example~\ref*{#1}}}
\newrefformat{fig}{\savehyperref{#1}{Figure~\ref*{#1}}}
\newrefformat{alg}{\savehyperref{#1}{Algorithm~\ref*{#1}}}
\newrefformat{rem}{\savehyperref{#1}{Remark~\ref*{#1}}}
\newrefformat{conj}{\savehyperref{#1}{Conjecture~\ref*{#1}}}
\newrefformat{prop}{\savehyperref{#1}{Proposition~\ref*{#1}}}
\newrefformat{proto}{\savehyperref{#1}{Protocol~\ref*{#1}}}
\newrefformat{prob}{\savehyperref{#1}{Problem~\ref*{#1}}}
\newrefformat{claim}{\savehyperref{#1}{Claim~\ref*{#1}}}
\newrefformat{que}{\savehyperref{#1}{Question~\ref*{#1}}}
\newrefformat{op}{\savehyperref{#1}{Open Problem~\ref*{#1}}}
\newrefformat{fn}{\savehyperref{#1}{Footnote~\ref*{#1}}}

\title{Policy Optimization in Adversarial MDPs: \\ Improved Exploration via Dilated Bonuses}

\author{%
  Haipeng Luo\thanks{Equal contribution.} \\
  \texttt{haipengl@usc.edu} \\ 
  \and
  Chen-Yu Wei\footnotemark[1] \\
  \texttt{chenyu.wei@usc.edu} \\
  \and
  Chung-Wei Lee \\
  \texttt{leechung@usc.edu} \\
  \and
  University of Southern California
}

\date{}

\begin{document}

\maketitle

\begin{abstract}
  Policy optimization is a widely-used method in reinforcement learning. Due to its local-search nature, however, theoretical guarantees on global optimality often rely on extra assumptions on the Markov Decision Processes (MDPs) that bypass the challenge of global exploration. To eliminate the need of such assumptions, in this work, we develop a general solution that adds \textit{dilated bonuses} to the policy update to facilitate global exploration. To showcase the power and generality of this technique, we apply it to several episodic MDP settings with adversarial losses and bandit feedback, improving and generalizing the state-of-the-art. Specifically, in the tabular case, we obtain $\otil(\sqrt{T})$ regret where $T$ is the number of episodes, improving the $\otil({T}^{\nicefrac{2}{3}})$ regret bound by~\cite{shani2020optimistic}. When the number of states is infinite, under the assumption that the state-action values are linear in some low-dimensional features, we obtain $\otil({T}^{\nicefrac{2}{3}})$ regret with the help of a simulator, matching the result of~\cite{neu2020online} while importantly removing the need of an exploratory policy that their algorithm requires. 
\ifnobug  
  When a simulator is unavailable, we further consider a linear MDP setting and obtain $\otil({T}^{\nicefrac{14}{15}})$ regret, which is the first result for linear MDPs with adversarial losses and bandit feedback. 
\fi  
\end{abstract}

\section{Introduction}\label{sec:intro}
Policy optimization methods are among the most widely-used methods in reinforcement learning. Its empirical success has been demonstrated in various domains such as computer games \citep{schulman2017proximal} and robotics \citep{levine2013guided}. However, due to its local-search nature, global optimality guarantees of policy optimization often rely on unrealistic assumptions to ensure global exploration (see e.g.,~\citep{abbasi2019politex, agarwal2020optimality, neu2020online, wei2021learning}), making it theoretically less appealing compared to other methods.

Motivated by this issue, a line of recent works~\citep{cai2020provably, shani2020optimistic, agarwal2020pc, zanette2021cautiously} equip policy optimization with global exploration by adding exploration bonuses to the update, and prove favorable guarantees even without making extra exploratory assumptions.
Moreover, they all demonstrate some robustness aspect of policy optimization (such as being able to handle adversarial losses or a certain degree of model mis-specification).
Despite these important progresses, however, many limitations still exist, including worse regret rates comparing to the best value-based or model-based approaches~\citep{shani2020optimistic, agarwal2020pc, zanette2021cautiously}, or requiring full-information feedback on the entire loss function (as opposed to the more realistic bandit feedback)~\citep{cai2020provably}.

To address these issues, in this work, we propose a new type of exploration bonuses called \emph{dilated bonuses}, which satisfies a certain \emph{dilated Bellman equation}
and provably leads to improved exploration compared to existing works (\pref{sec: general recipe}).
We apply this general idea to advance the state-of-the-art of policy optimization for learning finite-horizon episodic MDPs with \emph{adversarial losses and bandit feedback}.
More specifically, our main results are:

\begin{itemize}[leftmargin=1em]
  \setlength\itemsep{1em}
\item   
First, in the tabular setting, addressing the main open question left in~\citep{shani2020optimistic}, we improve their $\otil(T^{\nicefrac{2}{3}})$ regret to the optimal $\otil(\sqrt{T})$ regret.
This shows that policy optimization, which performs local optimization, is as capable as 
other occupancy-measure-based global optimization algorithms~\citep{jin2019learning, lee2020bias} in terms of global exploration.
Moreover, our algorithm is computationally more efficient than those global methods since they require solving some convex optimization in each episode. (\pref{sec: tabular})

\item
Second, to further deal with large-scale problems, we consider a linear function approximation setting where the state-action values are linear in some known low-dimensional features and also a simulator is available, the same setting considered by~\citep{neu2020online}.
\ifnobug
We obtain the same $\otil(T^{\nicefrac{2}{3}})$ regret while importantly removing their exploratory assumption. 
\fi
(\pref{sec: linear Q})

\item
\ifnobug
Finally, to remove the need of a sampling oracle, we further consider linear MDPs, a special case where the transition kernel is also linear in the features.
To our knowledge, the only existing works that consider adversarial losses in this setup are~\citep{cai2020provably}, which obtains $\otil(\sqrt{T})$ regret but requires full-information feedback on the loss functions,
and~\citep{neu2021online} (an updated version of~\citep{neu2020online}), which obtains $\otil(\sqrt{T})$ regret under bandit feedback but requires perfect knowledge of the transition as well as an exploratory assumption.
We propose the first algorithm for the most challenging setting with bandit feedback and unknown transition, which achieves $\otil(T^{\nicefrac{14}{15}})$ regret without any exploratory assumption. 
\fi
(\pref{sec: linear MDP})
\end{itemize}

\ifnobug
We emphasize that unlike the tabular setting (where we improve existing regret rates of policy optimization),
in the two adversarial linear function approximation settings with bandit feedback that we consider, researchers have not been able to show \textit{any} sublinear regret for policy optimization without exploratory assumptions before our work, which shows the critical role of our proposed dilated bonuses.
In fact, there are simply no existing algorithms with sublinear regret \textit{at all} for these two settings, be it policy-optimization-type or not.
This shows the advantage of policy optimization over other approaches, when combined with our dilated bonuses.
\fi

\paragraph{Related work.}
In the tabular setting, except for~\citep{shani2020optimistic}, most algorithms apply the occupancy-measure-based framework to handle adversarial losses (e.g.,~\citep{pmlr-v97-rosenberg19a, jin2019learning, chen2020minimax, chen2021finding}), which as mentioned is computationally expensive.
For stochastic losses, there are many more different approaches such as model-based ones~\citep{jaksch2010near, dann2015sample, azar2017minimax, fruit2018efficient, zanette2019tighter} and value-based ones~\citep{jin2018q, dong2019q}.

Theoretical studies for linear function approximation have gained increasing interest recently~\citep{yang2020reinforcement,  zanette2020frequentist, jin2020provably}. 
Most of them study stochastic/stationary losses, with the exception of~\citep{cai2020provably, neu2020online, neu2021online}.
Our algorithm for the linear MDP setting bears some similarity to those of \citep{agarwal2020pc, zanette2021cautiously} which consider stationary losses.
However, in each episode, their algorithms first execute an exploratory policy (from a  \emph{policy cover}), and then switch to the policy suggested by the policy optimization algorithm, which inevitably leads to linear regret when facing adversarial losses.  

\section{Problem Setting}\label{sec: settings}
We consider an MDP specified by a state space $X$ (possibly infinite), a finite action space $A$, and a transition function $P$ with $P(\cdot|x,a)$ specifying the distribution of the next state after taking action $a$ in state $x$. In particular, we focus on the \emph{finite-horizon episodic setting} in which $X$ admits a layer structure and can be partitioned into $X_0,X_1,\dots,X_H$ for some fixed parameter $H$, where $X_0$ contains only the initial state $x_0$, $X_H$ contains only the terminal state $x_H$, and for any $x\in X_h$, $h=0,\dots,H-1$, $P(\cdot|x,a)$ is supported on $X_{h+1}$ for all $a\in A$ (that is, transition is only possible from $X_h$ to $X_{h+1}$). An {episode} refers to a trajectory that starts from $x_0$ and ends at $x_H$ following some series of actions and the transition dynamic.  The MDP may be assigned with a loss function $\ell: X\times A\rightarrow [0,1]$ so that $\ell(x,a)$ specifies the loss suffered when selecting action $a$ in state $x$. 

A policy $\pi$ for the MDP is a mapping $X\rightarrow \Delta(A)$, where $\Delta(A)$ denotes the set of distributions over $A$ and $\pi(a|x)$ is the probability of choosing action $a$ in state $x$. Given a loss function $\ell$ and a policy $\pi$, the expected total loss of $\pi$ is given by 
$
    V^{\pi}(x_0; \ell) = \E\big[\sum_{h=0}^{H-1} \ell(x_h, a_h) ~\big|~ a_h \sim \pi_t(\cdot|x_h), x_{h+1}\sim P(\cdot|x_h, a_h)\big]. 
$
It can also be defined via the Bellman equation involving the \emph{state value function} $V^{\pi}(x;\ell)$ and the \emph{state-action value function} $Q^\pi(x,a;\ell)$ (a.k.a. $Q$-function) defined as below: 
$V(x_H;\ell)=0$,
\begin{align*}
    Q^\pi(x,a;\ell)=\ell(x,a)+\E_{x'\sim P(\cdot|x,a)}\left[V^\pi(x';\ell)\right],~\text{and}~V^\pi(x;\ell)=\E_{a\sim \pi(\cdot|x)}\left[Q^\pi(x,a;\ell)\right].
\end{align*}

We study online learning in such a finite-horizon MDP with \emph{unknown transition}, \emph{bandit feedback}, and \emph{adversarial losses}. The learning proceeds through $T$ episodes. Ahead of time, an adversary arbitrarily decides $T$ loss functions $\ell_1, \ldots, \ell_T$, without revealing them to the learner.  
Then in each episode $t$, the learner decides a policy $\pi_t$ based on all information received prior to this episode, executes $\pi_t$ starting from the initial state $x_0$,
generates and observes a trajectory $\{(x_{t,h}, a_{t,h}, \ell_t(x_{t,h}, a_{t,h}))\}_{h=0}^{H-1}$. 
Importantly, the learner does not observe any other information about $\ell_t$ (a.k.a. bandit feedback).\footnote{Full-information feedback, on the other hand, refers to the easier setting where the entire loss function $\ell_t$ is revealed to the learner at the end of episode $t$.}
The goal of the learner is to minimize the regret, defined as
\[
\Reg=\sum_{t=1}^TV^{\pi_t}_t(x_0)-\min_{\pi} \sum_{t=1}^T V^\pi_t(x_0),
\]
where we use $V_t^\pi(x)$ as a shorthand for
$V^\pi(x;\ell_t)$ (and similarly $Q_t^\pi(x,a)$ as a shorthand for $Q^\pi(x,a;\ell_t)$).
Without further structures, the best existing regret bound is $\otil(H|X|\sqrt{|A|T})$~\citep{jin2019learning}, with an extra $\sqrt{X}$ factor compared to the best existing lower bound~\citep{jin2018q}.

\paragraph{Occupancy measures. }
For a policy $\pi$ and a state $x$, we define $q^\pi(x)$ to be the probability (or probability measure when $|X|$ is infinite) of visiting state $x$ within an episode when following $\pi$. 
When it is necessary to highlight the dependence on the transition, we write it as $q^{P, \pi}(x)$.
Further define $q^{\pi}(x,a)=q^{\pi}(x)\pi(a|x)$ and $\QQ_t(x,a)=q^{\pi_t}(x,a)$.
Finally, we use $\qstar$ as a shorthand for $q^{\pi^\star}$ where $\pi^\star \in\argmin_{\pi} \sum_{t=1}^T V^\pi_t(x_0)$ is one of the optimal policies.

Note that by definition, we have $V^{\pi}(x_0; \ell) = \sum_{x,a} q^{\pi}(x,a)\ell(x,a)$.
In fact, we will overload the notation and let $V^{\pi}(x_0; b) = \sum_{x,a} q^{\pi}(x,a)b(x,a)$ for any function $b: X\times A \rightarrow \fR$ (even though it might not correspond to a real loss function).


\paragraph{Other notations. }
We denote by $\E_t[\cdot]$ and $\Var_t[\cdot]$ the expectation and variance conditioned on everything prior to episode $t$.
For a matrix $\cov$ and a vector $z$ (of appropriate dimension), $\|z\|_\cov$ denotes the quadratic norm $\sqrt{z^\top \cov z}$.
The notation $\otil(\cdot)$ hides all logarithmic factors. 





\section{Dilated Exploration Bonuses} 
\label{sec: general recipe}

In this section, we start with a general discussion on designing exploration bonuses (not specific to policy optimization), and then introduce our new dilated bonuses for policy optimization.
For simplicity, the exposition in this section assumes a finite state space, but the idea generalizes to an infinite state space.

When analyzing the regret of an algorithm, very often we run into the following form:
\begin{align}
    \Reg = \sum_{t=1}^T V^{\pi_t}_t(x_0) - \sum_{t=1}^T V^{\pistar}_t(x_0) 
    &\leq o(T) + 
    \sum_{t=1}^T \sum_{x, a} \qstar(x, a)\bonus_{t}(x,a)
    = o(T) + 
    \sum_{t=1}^T V^{\pistar}(x_0; \bonus_t),
    \label{eq: extra}
\end{align}
for some function $\bonus_{t}(x,a)$ usually related to some estimation error or variance  that can be prohibitively large.
For example, in policy optimization, the algorithm performs local search in each state essentially using a multi-armed bandit algorithm and treating $Q^{\pi_t}_t(x,a)$ as the loss of action $a$ in state $x$. 
Since $Q^{\pi_t}_t(x,a)$ is unknown, however, the algorithm has to use some estimator of $Q^{\pi_t}_t(x,a)$ instead, whose bias and variance both contribute to the $b_t$ function.
Usually, $b_t(x,a)$ is large for a rarely-visited state-action pair $(x,a)$ and is inversely related to $q_t(x,a)$, which is exactly why most analysis relies on the assumption that some \emph{distribution mismatch coefficient} related to $\nicefrac{\qstar(x,a)}{\QQ_t(x,a)}$ is bounded (see e.g.,~\citep{agarwal2020optimality, wei2020model}).

On the other hand, an important observation is that while $V^{\pistar}(x_0; \bonus_t)$ can be prohibitively large, its counterpart with respect to the learner's policy $V^{\pi_t}(x_0; \bonus_t)$ is usually nicely bounded.
For example, if $b_t(x,a)$ is inversely related to $\QQ_t(x,a)$ as mentioned, then $V^{\pi_t}(x_0; \bonus_t) = \sum_{x,a}\QQ_t(x,a)\bonus_{t}(x,a)$ is small no matter how small $\QQ_t(x,a)$ could be for some $(x,a)$.
This observation, together with the linearity property $V^{\pi}(x_0; \ell_t-\bonus_t) = V^{\pi}(x_0; \ell_t) - V^{\pi}(x_0; \bonus_t)$,
suggests that we treat $\ell_t-\bonus_t$ as the loss function of the problem, or in other words, add a (negative) bonus to each state-action pair, which intuitively encourages exploration due to underestimation.
Indeed, assuming for a moment that \pref{eq: extra} still roughly holds even if we treat $\ell_t-\bonus_t$ as the loss function:
\begin{align}
    &\sum_{t=1}^T V^{\pi_t}(x_0; \ell_t-\bonus_t) - \sum_{t=1}^T V^{\pistar}(x_0; \ell_t-\bonus_t)  \lesssim  o(T) + 
    \sum_{t=1}^T V^{\pistar}(x_0; \bonus_t).
    \label{eq: app regret bound}
\end{align}
Then by linearity and rearranging, we have
\begin{align}
    \Reg = \sum_{t=1}^T V^{\pi_t}_t(x_0) - \sum_{t=1}^T V^{\pistar}_t(x_0)  &\lesssim o(T) + 
    \sum_{t=1}^T V^{\pi_t}(x_0; \bonus_t). \label{eq: desired bound 2}
\end{align}
Due to the switch from $\pistar$ to $\pi_t$ in the last term compared to \pref{eq: extra}, this is usually enough to prove a desirable regret bound without making extra assumptions.

The caveat of this discussion is the assumption of \pref{eq: app regret bound}.
Indeed, after adding the bonuses, which itself contributes some more bias and variance, one should expect that $b_t$ on the right-hand side of \pref{eq: app regret bound} becomes something larger, breaking the desired cancellation effect to achieve \pref{eq: desired bound 2}.
Indeed, the definition of $b_t$ essentially becomes circular in this sense.

\paragraph{Dilated Bonuses for Policy Optimization}
To address this issue, we take a closer look at the policy optimization algorithm specifically.
As mentioned, policy optimization decomposes the problem into individual multi-armed bandit problems in each state and then performs local optimization.
This is based on the well-known performance difference lemma \citep{kakade2002approximately}:
\begin{align*}
\Reg = \sum_{x} \qstar(x) \sum_{t=1}^T \sum_a\Big(\pi_t(a|x) - \pistar(a|x)\Big)Q^{\pi_t}_t(x,a),
\end{align*}
showing that in each state $x$, the learner is facing a bandit problem with $Q^{\pi_t}_t(x,a)$ being the loss for action $a$.
Correspondingly, incorporating the bonuses $b_t$ for policy optimization means subtracting the bonus $Q^{\pi_t}(x,a;\bonus_t)$ from $Q^{\pi_t}_t(x,a)$ for each action $a$ in each state $x$.
Recall that $Q^{\pi_t}(x,a;\bonus_t)$ satisfies the Bellman equation $Q^{\pi_t}(x,a;\bonus_t) = \bonus_t(x,a) + \E_{x'\sim P(\cdot|x,a)}\E_{a'\sim \pi_t(\cdot|x')}\left[\bonusQ_t(x',a')\right]$.
To resolve the issue mentioned earlier, we propose to replace this bonus function $Q^{\pi_t}(x,a;\bonus_t)$ with its \emph{dilated} version $B_t(s,a)$ satisfying the following \emph{dilated Bellman equation}:
\begin{align}
    \bonusQ_t(x,a) = \bonus_t(x,a) + \left(1+\frac{1}{H}\right)\E_{x'\sim P(\cdot|x,a)}\E_{a'\sim \pi_t(\cdot|x')}\left[\bonusQ_t(x',a')\right] \label{eq: dilated}
\end{align}
(with $\bonusQ_t(x_H,a) = 0$ for all $a$).
The only difference compared to the standard Bellman equation is the extra $(1+\frac{1}{H})$ factor, which slightly increases the weight for deeper layers and thus intuitively induces more exploration for those layers.
Due to the extra bonus compared to $Q^{\pi_t}(x,a;\bonus_t)$, the regret bound also increases accordingly. 
In all our applications, this extra amount of regret turns out to be of the form $\frac{1}{H}\sum_{t=1}^T \sum_{x,a} \qstar(x)\pi_t(a|x)\bonusQ_t(x,a)$, leading to
\begin{align}
    &\sum_{x} \qstar(x)\sum_{t=1}^T  \sum_a \Big(\pi_t(a|x) - \pistar(a|x)\Big)\Big(Q^{\pi_t}_t(x,a) - \bonusQ_t(x,a) \Big) \nonumber \\
    &\qquad \qquad \leq o(T) + 
    \sum_{t=1}^T V^{\pistar}(x_0; \bonus_t) + \frac{1}{H}\sum_{t=1}^T \sum_{x,a} \qstar(x)\pi_t(a|x)\bonusQ_t(x,a). \label{eq:modified bound} 
\end{align}
With some direct calculation, one can show that this is enough to show a regret bound that is only a constant factor larger than the desired bound in \pref{eq: desired bound 2}!
This is summarized in the following lemma.
\begin{lemma}\label{lem:key}
If \pref{eq:modified bound} holds with $B_t$ defined in \pref{eq: dilated},
then $\Reg \leq o(T) + 3\sum_{t=1}^T V^{\pi_t}(x_0; \bonus_t)$.  
\end{lemma}

The high-level idea of the proof is to show that the bonuses added to a layer $h$ is enough to cancel the large bias/variance term (including those coming from the bonus itself) from layer $h+1$.
Therefore, cancellation happens in a layer-by-layer manner except for layer $0$,
where the total amount of bonus can be shown to be at most $(1+\frac{1}{H})^H\sum_{t=1}^T V^{\pi_t}(x_0; \bonus_t) \leq 3\sum_{t=1}^T V^{\pi_t}(x_0; \bonus_t)$. 

Recalling again that $V^{\pi_t}(x_0; \bonus_t)$ is usually nicely bounded, 
we thus arrive at a favorable regret guarantee without making extra assumptions.
Of course, since the transition is unknown, we cannot compute $B_t$ exactly.
However, \pref{lem:key} is robust enough to handle either a good approximate version of $B_t$ (see \pref{lem: lemma for the condition}) or a version where \pref{eq: dilated} and \pref{eq:modified bound} only hold in expectation (see \pref{lem: expected version}), which is enough for us to handle unknown transition. 
In the next three sections, we apply this general idea to different settings, showing what $b_t$ and $B_t$ are concretely in each case. 

\section{The Tabular Case}
\label{sec: tabular}

In this section, we study the tabular case where the number of states is finite. 
We propose a policy optimization algorithm with $\otil(\sqrt{T})$ regret, improving the $\otil(T^{\nicefrac{2}{3}})$ regret of~\citep{shani2020optimistic}.
See \pref{alg:MDP-traj} for the complete pseudocode.

\setcounter{AlgoLine}{0}
\begin{savenotes}
\begin{algorithm}[hbt!] 
	\caption{Policy Optimization with Dilated Bonuses (Tabular Case)}
	\label{alg:MDP-traj}
	\textbf{Parameters:} $\delta\in (0,1)$, $\eta=\min\left\{\nicefrac{1}{24 H^3}, \nicefrac{1}{\sqrt{|X||A|HT}}\right\}$, \ $\gamma=2\eta H$. 
	
	\textbf{Initialization:} 
	Set epoch index $k=1$ and confidence set $\calP_1$ as the set of all transition functions. 
	For all $(x,a,x')$, initialize counters
$N_0(x,a)=N_1(x,a)=0, N_0(x,a,x')=N_1(x,a,x')=0$.
	 
	 \For{ $t = 1 ,2,\dots, T$ }{
	    \textbf{Step 1: Compute and execute policy.} Execute $\pi_t$ for one episode, where
	\begin{equation}\label{eq:MDP_log_barrier}
		\pi_{t}(a|x) \propto  
		\exp\left(-\eta\sum_{\tau=1}^{t-1}\left(\Qht_\tau(x,a) - \bonusQ_\tau(x,a)\right)\right), 
	\end{equation}
	and obtain trajectory $\{(x_{t,h}, a_{t,h}, \ell_t(x_{t,h}, a_{t,h}))\}_{h=0}^{\Lyr-1}$. \\
	
		\textbf{Step 2: Construct $Q$-function estimators.} For all $h\in\{0,\ldots, H-1\}$ and $(x,a)\in X_h\times A$, 
		\begin{align}\label{eq:MDP_estimator} \Qht_t(x,a) = \frac{L_{t,h}}{\Ut(x,a)+\gamma}\ind_t(x,a),  
		\end{align}
		with
		$ L_{t,h} = \sum\limits_{i=h}^{\Lyr-1}\ell_t(x_{t,i},a_{t,i}),  
		\Ut(x,a) = \max\limits_{\Pht \in \mathcal{P}_k} q^{\Pht, \pi_t}(x,a)$, 
		and $\ind_t(x,a) = \ind\{x_{t,h} = x, a_{t,h} = a\}.$  \\
	
		\textbf{Step 3: Construct bonus functions.} For all $(x,a)\in X\times A$, 
		\begin{align}
		\bonus_t(x) &= \E_{a\sim \pi_t(\cdot|x)}  \left[\frac{3\gamma H + H(\Ut(x,a)-\Lt(x,a))}{\Ut(x, a) + \gamma}\right]    \label{eq: err function definition}\\
		 \bonusQ_t(x,a) &=
		 \displaystyle\bonus_t(x) +\left(1+\frac{1}{H}\right) \max_{\Pht\in\calP_k}\E_{x'\sim \Pht(\cdot|x,a)}\E_{a'\sim \pi_t(\cdot|x')}\left[\bonusQ_t(x', a')\right]  
		 \label{eq: bonus bellman 1} 
		\end{align}
		where $\Lt(x,a)= \min_{\Pht \in \mathcal{P}_k} q^{\Pht, \pi_t}(x,a)$ and $\bonusQ_t(x_H,a) =  0$ for all $a$. \\

		\textbf{Step 4: Update model estimation.} $\forall h<H$,
		$
		N_k(x_{t,h}, a_{t,h}) \overset{+}{\leftarrow} 1,  N_k(x_{t,h}, a_{t,h}, x_{t, h+1}) \overset{+}{\leftarrow} 1.
		$\footnote{We use $y \overset{+}{\leftarrow} z$ as a shorthand for the increment operation $y \leftarrow y + z$.}

		 \If {$\exists h, \  N_k(x_{t,h}, a_{t,h}) \geq \max\{1, 2N_{k-1}(x_{t,h},a_{t,h})\}$}{
			
			Increment epoch index $k \overset{+}{\leftarrow} 1$ and 
			copy counters:  $
			    N_k \gets N_{k-1},  N_k \gets N_{k-1}$.
			

			Compute empirical transition $\Pbar_k(x'|x,a)=\frac{N_k(x,a,x')}{\max\left\{1,N_k(x,a)\right\}}$  and confidence set:
			\begin{equation}\label{eq:confset}
			\begin{split}
			\textstyle\mathcal{P}_k = &\Big\{\Pht:\left|\Pht(x'|x,a)-\Pbar_k(x'|x,a)\right|\le \conf_k(x'|x,a),\\ 
			&\quad\forall (x,a,x')\in X_h\times A\times X_{h+1}, h=0,1,\dots, \Lyr-1\Big\},
			\end{split}
			\end{equation}
			\begin{align*} 
			\text{where}
			\quad\conf_k(x'|x,a)= 4\sqrt{\frac{\Pbar_k(x'|x,a)\ln \left(\frac{T|X||A|}{\delta}\right)}{\max\{1,N_k(x,a)\}}}+\frac{28\ln \left(\frac{T|X||A|}{\delta}\right)}{3\max\{1, N_k(x,a)\}}. 
			\end{align*} 
		}		
	}
\end{algorithm}
\end{savenotes}

\paragraph{Algorithm design.}
First, to handle unknown transition, we follow the common practice (dating back to~\citep{jaksch2010near}) to maintain a confidence set of the transition, which is updated whenever the visitation count of a certain state-action pair is doubled.
We call the period between two model updates an epoch, and use $\calP_k$ to denote the confidence set for epoch $k$, formally defined in \pref{eq:confset}.

In episode $t$, the policy $\pi_t$ is defined via the standard multiplicative weight algorithm (also connected to Natural Policy Gradient~\citep{kakade2001natural, agarwal2020optimality, wei2021learning}), 
but importantly with the dilated bonuses incorporated such that 
$\pi_{t}(a|x) \propto \exp(-\eta\sum_{\tau=1}^{t-1}(\Qht_\tau(x,a) - \bonusQ_\tau(x,a)))$.
Here, $\eta$ is a step size parameter, $\Qht_\tau(x,a)$ is an importance-weighted estimator for $Q^{\pi_\tau}_\tau(x,a)$ defined in \pref{eq:MDP_estimator}, and $\bonusQ_\tau(x,a)$ is the dilated bonus defined in \pref{eq: bonus bellman 1}.

More specifically, for a state $x$ in layer $h$, $\Qht_t(x,a)$ is defined as $\frac{L_{t,h}\ind_t(x,a)}{\Ut(x,a)+\gamma}$, where $\ind_t(x,a)$ is the indicator of whether $(x,a)$ is visited during episode $t$; $L_{t,h}$ is the total loss suffered by the learner starting from layer $h$ till the end of the episode; $\Ut(x,a)=\max_{\Pht \in \mathcal{P}_k} q^{\Pht, \pi_t}(x,a)$ is the largest plausible value of $\QQ_t(x,a)$ within the confidence set, which can be computed efficiently using the \textsc{Comp-UOB} procedure of~\citep{jin2019learning} (see also \pref{app: omitted procedure}); and finally $\gamma$ is a parameter used to control the maximum magnitude of $\Qht_t(x,a)$.
To get a sense of this estimator, consider the special case when $\gamma=0$ and the transition is known so that we can set $\calP_k = \{P\}$ and thus $\Ut = \QQ_t$.
Then, since the expectation of $L_{t,h}$ conditioned on $(x,a)$ being visited is $Q^{\pi_t}_t(x,a)$ and the expectation of $\ind_t(x,a)$ is $q_t(x,a)$, we know that $\Qht_t(x,a)$ is an unbiased estimator for $Q^{\pi_t}_t(x,a)$.
The extra complication is simply due to the transition being unknown, forcing us to use $\Ut$ and $\gamma >0$ to make sure that $\Qht_t(x,a)$ is an optimistic underestimator, an idea similar to~\citep{jin2019learning}.

Next, we explain the design of the dilated bonus $B_t$.
Following the discussions of \pref{sec: general recipe},
we first figure out what the corresponding $\bonus_t$ function is in \pref{eq: extra}, by analyzing the regret bound without using any bonuses.
The concrete form of $\bonus_t$ turns out to be \pref{eq: err function definition},
whose value at $(x,a)$ is independent of $a$ and thus written as $\bonus_t(x)$ for simplicity.
Note that \pref{eq: err function definition} depends on the occupancy measure lower bound $\Lt(s,a) = \min_{\Pht \in \mathcal{P}_k} q^{\Pht, \pi_t}(x,a)$, the opposite of $\Ut(s,a)$, which can also be computed efficiently using a procedure similar to \textsc{Comp-UOB} (see \pref{app: omitted procedure}).
Once again, to get a sense of this, consider the special case with a known transition so that we can set $\calP_k = \{P\}$ and thus $\Ut = \Lt = \QQ_t$.
Then, one see that $\bonus_t(x)$ is simply upper bounded by $\E_{a\sim \pi_t(\cdot|x)}  \left[\nicefrac{3\gamma H}{\QQ_t(x, a)}\right] = \nicefrac{3\gamma H|A|}{\QQ_t(x)}$, which is inversely related to the probability of visiting state $x$, matching the intuition we provided in \pref{sec: general recipe} (that $\bonus_t(x)$ is large if $x$ is rarely visited).
The extra complication of \pref{eq: err function definition} is again just due to the unknown transition.

With $\bonus_t(x)$ ready, the final form of the dilated bonus $B_t$ is defined following the dilated Bellman equation of \pref{eq: dilated}, except that since $P$ is unknown, we once again apply optimism and find the largest possible value within the confidence set (see \pref{eq: bonus bellman 1}).
This can again be efficiently computed; see \pref{app: omitted procedure}.
This concludes the complete algorithm design.

\paragraph{Regret analysis.}
The regret guarantee of \pref{alg:MDP-traj} is presented below:

\begin{theorem}\label{thm:regret bound main}
	\pref{alg:MDP-traj} ensures that with probability $1-\order(\delta)$, $
	\Reg = \otil\left(H^2|X|\sqrt{AT} + H^4\right)$. 
\end{theorem}

Again, this improves the $\otil(T^{\nicefrac{2}{3}})$ regret of~\citep{shani2020optimistic}.
It almost matches the best existing upper bound for this problem, which is $\otil(H|X|\sqrt{|A|T})$~\citep{jin2019learning}.
While it is unclear to us whether this small gap can be closed using policy optimization, we point out that our algorithm is arguably more efficient than that of \citep{jin2019learning}, which performs global convex optimization over the set of all plausible occupancy measures in each episode.

The complete proof of this theorem is deferred to \pref{app: tabular appendix}.
Here, we only sketch an outline of proving \pref{eq:modified bound}, which, according to the discussions in \pref{sec: general recipe}, is the most important part of the analysis.
Specifically, we decompose the left-hand side of \pref{eq:modified bound}, $\sum_{x}\qstar(x)\sum_{t}\inner{\pi_t(\cdot|x)-\pistar(\cdot|x), \Q_t^{\pi_t}(x,\cdot)-\bonusQ_t(x,\cdot)}$, as $\bias + \biastwo + \regterm$, where
\begin{itemize}[leftmargin=1em]
  \setlength\itemsep{0em}
\item   $\bias = \sum_x\qstar(x)\sum_{t}\inn{\pi_t(\cdot|x), \Q_t^{\pi_t}(x,\cdot)-\Qht_t(x,\cdot)}$ measures the amount of underestimation of $\Qht_t$ related to $\pi_t$, which can be bounded by $\sum_{t}\ \sum_{x,a}\qstar(x)\pi_{t}(a| x)\Big(\frac{2\gamma H+ H(\Ut(x, a)-\Lt(x, a))}{\Ut(x, a)+\gamma}\Big)+\otil\left(\nicefrac{H}{\eta}\right)$ with high probability (\pref{lem: bias1});

\item $\biastwo = \sum_x\qstar(x)\sum_{t}\inn{\pistar(\cdot|x), \Qht_t(x,\cdot)-\Q_t^{\pi_t}(x,\cdot)}$ measures the amount of overestimation of $\Qht_t$ related to $\pistar$, which can be bounded by $\otil\left(\nicefrac{H}{\eta}\right)$ since $\Qht_t$ is an underestimator (\pref{lem: bias2});

\item $\regterm = \sum_x\qstar(x)\sum_{t}\inn{\pi_t(\cdot|x)-\pistar(\cdot|x), \Qht_t(x,\cdot)-B_t(x,\cdot)}$ is directly controlled by the multiplicative weight update, and is bounded by $\sum_{t} \sum_{x,a}\qstar(x)\pi_t(a|x) \left(\frac{\gamma H}{\Ut(x,a)+\gamma}+\frac{\bonusQ_t(x,a)}{H}\right) + \otil\left(\nicefrac{H}{\eta}\right)$ with high probability (\pref{lem: reg}).
\end{itemize}

Combining all with the definition of $\bonus_t$ proves the key \pref{eq:modified bound} (with the $o(T)$ term being $\otil(\nicefrac{H}{\eta})$).

\section{The Linear-\texorpdfstring{$Q$}{} Case}\label{sec: linear Q}
In this section, we move on to the more challenging setting where the number of states might be infinite, and function approximation is used to generalize the learner's experience to unseen states. We consider the most basic linear function approximation scheme where for any $\pi$, the $Q$-function $Q_t^{\pi}(x,a)$ is linear in some known feature vector $\phi(x,a)$, formally stated below.

\begin{assumption}[Linear-$Q$]\label{assum: linear Q assumption}
    Let $\phi(x,a)\in \mathbb{R}^d$ be a known feature vector of the state-action pair $(x,a)$. We assume that for any episode $t$, policy $\pi$, and layer $h$, there exists an unknown weight vector $\theta^{\pi}_{t,h}\in \mathbb{R}^d$ such that for all $(x, a)\in X_h\times A$,  $Q^{\pi}_t(x,a)=\phi(x,a)^\top \theta_{t,h}^{\pi}$. 
     Without loss of generality, we assume $\|\phi(x,a)\|\leq 1$ for all $(x,a)$ and $\|\theta_{t,h}^{\pi}\|\leq \sqrt{d}H$ for all $t,h,\pi$.  
\end{assumption}

For justification on the last condition on norms,
see~\citep[Lemma 8]{wei2021learning}.
This linear-$Q$ assumption has been made in several recent works with stationary losses~\citep{abbasi2019politex, wei2021learning} and also in~\citep{neu2020online} with the same adversarial losses.\footnote{The assumption in~\citep{neu2020online} is stated slightly differently (e.g., their feature vectors are independent of the action). However, it is straightforward to verify that the two versions are equivalent.}
It is weaker than the linear MDP assumption (see \pref{sec: linear MDP}) as it does not pose explicit structure requirements on the loss and transition functions.
Due to this generality, however, our algorithm also requires access to a \emph{simulator} to obtain samples drawn from the transition, formally stated below.



\begin{assumption}[Simulator]\label{assum: sampling oracle} The learner has access to a simulator, which takes a state-action pair $(x,a)\in X\times A$ as input, and generates a random outcome of the next state $x'\sim P(\cdot|x,a)$. \end{assumption}

Note that this assumption is also made by~\citep{neu2020online} and more earlier works with stationary losses (see e.g.,~\citep{azar2012sample,sidford2018near}).\footnote{The simulator required by~\cite{neu2020online} is in fact slightly weaker than ours and those from earlier works --- it only needs to be able to generate a trajectory starting from $x_0$ for any policy.}
In this setting, we propose a new policy optimization algorithm with $\otil(T^{\nicefrac{2}{3}})$ regret. 
See \pref{alg: linear Q} for the pseudocode.

\setcounter{AlgoLine}{0}
\DontPrintSemicolon 
\begin{algorithm}[t]
    \caption{Policy Optimization with Dilated Bonuses (Linear-$Q$ Case)}
    \label{alg: linear Q}
    \textbf{parameters}: $\gamma, \beta, \eta, \epsilon\in(0,\frac{1}{2})$,  $M=\left\lceil\frac{24\ln(dHT)}{\epsilon^2\gamma^2}\right\rceil$, $N=\left\lceil\frac{2}{\gamma}\ln \frac{1}{\epsilon \gamma}\right\rceil$.  \\ 
    \For{$t=1,2,\ldots, T$}{
        \textbf{Step 1: Interact with the environment.} 
        Execute $\pi_t$, which is defined such that for each $x\in X_h$,
        \begin{align}
              \pi_t(a|x) \propto \exp\left( -\eta \sum_{\tau=1}^{t-1}\left(\phi(x,a)^\top \hattheta_{\tau, h}   - \Bonus(\tau,x,a)\right) \right), \label{eq: linear Q policy} 
        \end{align} 
        and obtain trajectory $\{(x_{t,h}, a_{t,h}, \ell_t(x_{t,h}, a_{t,h}))\}_{h=0}^{H-1}$.  
        \ \\
       
        \textbf{Step 2: Construct covariance matrix inverse estimators.} 
        Collect $MN$ trajectories using the simulator and $\pi_t$. Let $\calT_t$ 
        be the set of trajectories.  Compute
        \begin{align*} 
             \left\{\hatcov_{t, h}\right\}_{h=0}^{H-1} = \GR\left(\calT_t, M, N, \gamma\right).   \tag{see \pref{alg: GR}} 
        \end{align*}
        \textbf{Step 3: Construct $Q$-function weight estimators.} For $h=0, \ldots, H-1$, compute 
        \begin{align}
            \hattheta_{t, h}&= \hatcov_{t,h} \phi(x_{t,h},a_{t,h})L_{t,h}, \qquad \text{where\ } L_{t,h}= \sum_{i=h}^{H-1}\ell_t(x_{t,i},a_{t,i}).\label{eq: theta estimator in linear Q}
        \end{align}
    }
\end{algorithm}
\begin{algorithm}[t]
    \caption{$\Bonus(t,x,a)$}
    \label{alg: generating B samples}

    \If{$\Bonus(t,x,a)$ has been called before}{\textbf{return} the value of $\Bonus(t,x,a)$ calculated last time.}
    
     Let $h$ be such that $x\in X_h$. 
    \lIf{$h=H$}{
        \textbf{return} $0$.
    }  
    Compute $\pi_t(\cdot|x)$, defined in \pref{eq: linear Q policy} (which involves recursive calls to $\Bonus$ for smaller $t$).\\  
    Get a sample of the next state $x'\leftarrow \oracle(x,a)$.  \\
    Compute $\pi_t(\cdot|x')$ (again, defined in \pref{eq: linear Q policy}), and sample an action $a'\sim \pi_t(\cdot|x')$.  
    
    \textbf{return} $\beta \|\phi(x,a)\|_{\hatcov_{t,h}}^2 +  \E_{j\sim \pi_t(\cdot|x)}\Big[\beta\|\phi(x,j)\|_{\hatcov_{t,h}}^2\Big] + \left(1+\frac{1}{H}\right)\Bonus(t, x',a')$. 
\end{algorithm}

\begin{algorithm}[t]
    \caption{$\GR(\calT, M, N, \gamma)$}
    \label{alg: GR}
    Denote the $MN$ trajectories in $\mathcal{T}$ as: $\{(x_{i,0}, a_{i,0}, \ldots, x_{i,H-1}, a_{i,H-1})\}_{i=1, \ldots, MN}$.
    Let $c=\frac{1}{2}$. \\
    \For{$m=1, \ldots, M$}{
    \For{$n=1,\ldots, N$}{
            $i=(m-1)N+n$. \\
            For all $h$, compute $Y_{n,h}=\gamma I +  \phi(x_{i,h},a_{i,h})\phi(x_{i,h},a_{i,h})^\top$.\\
            For all $h$, compute 
            $Z_{n,h}=\Pi_{j=1}^{n}(I-cY_{j,h})$.
        }
        For all $h$, set $\hatcovk_{h} = cI + c\sum_{n=1}^N Z_{n,h}$.   
    }
    For all $h$, set $\hatcov_{h} = \frac{1}{M}\sum_{m=1}^M \hatcovk_{h}$.\\
    \textbf{return} $\hatcov_{h}$ for all $h=0,\ldots, H-1$.   
\end{algorithm}



\paragraph{Algorithm design.}
The algorithm still follows the multiplicative weight update \pref{eq: linear Q policy} in each state $x\in X_h$ (for some $h$), but now with $\phi(x,a)^\top \hattheta_{t, h}$ as an estimator for $\Q_t^{\pi_t}(x,a) = \phi(x,a)^\top \theta_{t,h}^{\pi_t}$, and $\Bonus(t,x,a)$ as the dilated bonus $B_t(x,a)$.
Specifically, the construction of the weight estimator $\hattheta_{t, h}$ follows the idea of \citep{neu2020online} (which itself is based on the linear bandit literature) and is defined in \pref{eq: theta estimator in linear Q} as $\hatcov_{t,h}  \phi(x_{t,h},a_{t,h})L_{t,h}$.
Here, $\hatcov_{t,h}$ is an $\epsilon$-accurate estimator of $\left(\gamma I + \cov_{t,h}\right)^{-1}$, where $\gamma$ is a small parameter and $\cov_{t,h} = \E_t[\phi(x_{t,h},a_{t,h})\phi(x_{t,h},a_{t,h})^\top]$ is the covariance matrix for layer $h$ under policy $\pi_t$;
$L_{t,h}= \sum_{i=h}^{H-1}\ell_t(x_{t,i},a_{t,i})$ is again the loss suffered by the learner starting from layer $h$, whose conditional expectation is $\Q_t^{\pi_t}(x_{t,h},a_{t,h}) = \phi(x_{t,h},a_{t,h})^\top \theta_{t,h}^{\pi_t}$.
Therefore, when $\gamma$ and $\epsilon$ approach $0$, one see that
$\hattheta_{t, h}$ is indeed an unbiased estimator of $\theta_{t,h}^{\pi_t}$.
We adopt the \GR procedure (see \pref{alg: GR}) of \citep{neu2020online} to compute $\hatcov_{t,h}$, which requires calling the simulator multiple times. 

Next, we explain the design of the dilated bonus. 
Again, following the general principle discussed in \pref{sec: general recipe},
we identify $b_t(x,a)$ in this case as $\beta \|\phi(x,a)\|_{\hatcov_{t,h}}^2 +  \E_{j\sim \pi_t(\cdot|x)}\big[\beta\|\phi(x,j)\|_{\hatcov_{t,h}}^2\big]$ for some parameter $\beta > 0$.
Further following the dilated Bellman equation \pref{eq: dilated}, we thus define $\Bonus(t,x,a)$ recursively as the last line of \pref{alg: generating B samples},
where we replace the expectation $\E_{(x',a')}[\Bonus(t,x',a')]$ with one single sample for efficient implementation.

However, even more care is needed to actually implement the algorithm.
First, since the state space is potentially infinite, one cannot actually calculate and store the value of $\Bonus(t,x,a)$ for all $(x,a)$, but can only calculate them on-the-fly when needed.
Moreover, unlike the estimators for $\Q_t^{\pi_t}(x,a)$, which can be succinctly represented and stored via the weight estimator $\hattheta_{t, h}$,
this is not possible for $\Bonus(t,x,a)$ due to the lack of any structure.
Even worse, the definition of $\Bonus(t,x,a)$ itself depends on $\pi_t(\cdot | x)$ and also $\pi_t(\cdot | x')$ for the afterstate $x'$, which, according to \pref{eq: linear Q policy}, further depends on $\Bonus(\tau,x,a)$ for $\tau < t$, resulting in a complicated recursive structure.
This is also why we present it as a procedure in \pref{alg: generating B samples} (instead of $B_t(x,a)$).
In total, this leads to $(TAH)^{\order(H)}$ number of calls to the simulator. 
Whether this can be improved is left as a future direction.

\paragraph{Regret guarantee}
By showing that \pref{eq:modified bound} holds in expectation for our algorithm, we obtain the following regret guarantee. (See \pref{app: linear Q appendix} for the proof.)

\begin{theorem}\label{thm: linear Q theorem}
    Under \pref{assum: linear Q assumption} and \pref{assum: sampling oracle}, with appropriate choices of the parameters $\gamma, \beta, \eta, \epsilon$, \pref{alg: linear Q} ensures $\E[\Reg]=\otil\left(H^2(dT)^{\nicefrac{2}{3}}\right)$ (the dependence on $|A|$ is only logarithmic). 
\end{theorem}   

This matches the $\otil(T^{\nicefrac{2}{3}})$ regret of~\citep[
Theorem~1]{neu2020online}, without the need of their assumption which essentially says that the learner is given an exploratory policy to start with.\footnote{Under an even stronger assumption that every policy is exploratory, they also improve the regret to $\otil(\sqrt{T})$; see~\citep[
Theorem~2]{neu2020online}.}
To our knowledge, this is the first no-regret algorithm for the linear-$Q$ setting (with adversarial losses and bandit feedback) when no exploratory assumptions are made.
\ifnobug
\fi


\ifnobug
\section{The Linear MDP Case}\label{sec: linear MDP}

To remove the need of a simulator, we further consider the linear MDP case, a special case of the linear-$Q$ setting.
It is equivalent to \pref{assum: linear Q assumption} plus the extra assumption that the transition function also has a low-rank structure, formally stated below.
\fi

\begin{assumption}[Linear MDP]\label{assum: linear MDP assumption}
    The MDP satisfies \pref{assum: linear Q assumption} and that for any $h$ and $x'\in X_{h+1}$, there exists an unknown weight vector $\nu^{x'}_{h}\in \mathbb{R}^d$ such that $P(x'|x,a) = \phi(x,a)^\top \nu^{x'}_{h}$ for all $(x,a)\in X_h\times A$. 
\end{assumption}

There is a surge of works studying this setting, with~\citep{cai2020provably} being the closest to us.
They achieve $\otil(\sqrt{T})$ regret but require full-information feedback of the loss functions, and there are no existing results for the bandit feedback setting, except for a concurrent work~\citep{neu2021online} which assumes perfect knowledge of the transition and an exploratory condition.
We propose the first algorithm with sublinear regret for this problem with unknown transition and bandit feedback, shown in \pref{alg:linearMDP}.
The structure of \pref{alg:linearMDP} is similar to that of \pref{alg: linear Q},
but importantly with the following modifications.

\begin{savenotes}
\begin{algorithm}[hbt!] 
    \caption{Policy Optimization with Dilated Bonuses (Linear MDP Case)}\label{alg:linearMDP}
    \textbf{Parameters}: $\gamma, \beta, \eta, \epsilon, \explore\in(0,\frac{1}{2}), \delta$, $M=\left\lceil\frac{96\ln(dHT)}{\epsilon^2\gamma^2}\right\rceil$, $N=\left\lceil\frac{2}{\gamma}\ln \frac{1}{\epsilon \gamma}\right\rceil$, $\len=2MN$, $\alpha=\frac{\explore}{6\beta}$, $M_0 = \left\lceil\alpha^2 dH^2\right\rceil$, $N_0 = \frac{100M_0^4\log(T/\delta)}{\alpha^2}$, $T_0=M_0 N_0$.  \\  
    
    Construct a mixture policy $\pcov$ and its estimated covariance matrices (which requires interacting with the environment for the first $T_0$ rounds using \pref{alg: find policy cover}):
\[
    \pcov, \left\{\widehat{\cov}^{\textcov}_h\right\}_{h=0, \ldots, H-1} \leftarrow \findpcov(M_0, N_0, \alpha, \delta). 
\]
    
    Define known state set $\calK = \left\{x \in X: \forall a\in A, \|\phi(x,a)\|_{(\widehat{\cov}_h^{\textcov})^{-1}}^2 \leq \alpha \text{ where $h$ is such that $x \in X_h$} \right\}$.
    
   
    \For{$k=1,2,\ldots, (T-T_0)/\len$}{
        \textbf{Step 1: Interact with the environment.} Define $\pi_k$ as the following: for $x\in X_h$,
        \begin{align}\label{eq:linear mdp policy}
            \pi_k(a|x) \propto \exp\left( -\eta \sum_{\tau=1}^{k-1} \left(\phi(x,a)^\top \hattheta_{\tau, h} - \phi(x,a)^\top\bonusVec_{\tau,h} - \bonus_\tau(x,a)\right) \right)  
        \end{align}
        where  
        $
            \bonus_\tau(x,a) = \left(\beta\|\phi(x,a)\|_{\hatcov_{\tau, h}}^2 + \beta \E_{a' \sim \pi_\tau(\cdot|x)}\Big[\|\phi(x,a')\|_{\hatcov_{\tau,h}}^2\Big]\right)\ind[x \in \calK] 
       $. 
       
        Randomly partition $\{T_0+(k-1)\len+1, \ldots, T_0 + k\len\}$ into two parts: $S$ and $S'$, such that $|S|=|S'|=\len/2$. 
        
        \For{$t=T_0 + (k-1)\len+1, \ldots, T_0 + k\len$}{
            Draw $Y_t\sim \textsc{Bernoulli}(\explore)$.
            \\
            \If{$Y_t=1$}{
                \lIf{$t\in S$}{Execute $\pcov$.}
                \lElse{
                Draw $h_t^*\stackrel{\text{unif.}}{\sim} \{0, \ldots, H-1\}$;
                execute $\pcov$ in steps $0, \ldots, h_t^*-1$ and $\pi_k$ in steps $h_t^*, \ldots, H-1$. 
                }
            }
            \lElse{Execute $\pi_k$.}
            Collect trajectory $\{(x_{t,h}, a_{t,h}, \ell_t(x_{t,h},a_{t,h}))\}_{h=0}^{H-1}$. 
        }
        \ \\
       \textbf{Step 2: Construct inverse covariance matrix estimators.} 
       Let 
       \begin{align}
           \mathcal{T}_k &= \{(x_{t,0}, a_{t,0}, \ldots, x_{t,H-1}, a_{t,H-1})\}_{t\in S}, \tag{the trajectories in $S$} \\
           \left\{\hatcov_{k,h}\right\}_{h=0}^{H-1} &= \GR(\mathcal{T}_k, M, N, \gamma).  \label{eq: linear mdp gr}  
       \end{align}

       \textbf{Step 3: Construct $Q$-function weight estimators.} Computer for all $h$ (with $L_{t,h} = \sum_{i=h}^{H-1}\ell_t(x_{t,i},a_{t,i})$):
        \begin{align}
            \hattheta_{k,h}&= \hatcov_{k,h}\left(\frac{1}{|S'|}\sum_{t\in S'} ((1-Y_t) + Y_t H\ind[h=h_t^*]) \phi(x_{t,h},a_{t,h})L_{t,h} \right). \label{eq: theta estimator in linear MDP}
        \end{align}
        
        \textbf{Step 4: Construct bonus function weight estimators.} Computer for all $h$ :
        \begin{align}
             \bonusVec_{k,h}&= \hatcov_{k,h}\left( \frac{1}{|S'|}\sum_{t\in S'}((1-Y_t) + Y_t H\ind[h=h_t^*])\phi(x_{t,h},a_{t,h})D_{t,h} \right), \label{eq: bonus function estimator in linear MDP}
        \end{align} 
        where 
        $D_{t,h} = \sum_{i=h+1}^{H-1}\left(1+\frac{1}{H}\right)^{i-h} \bonus_{k}(x_{t,i}, a_{t,i})$.
    } 
\end{algorithm}
\end{savenotes}

\paragraph{A succinct representation of dilated bonuses}
Our definition of $b_t$ remains the same as in the linear-$Q$ case.
However, due to the low-rank transition structure in linear MDPs, we are now able to efficiently construct estimators of $\bonusQ_t(x,a)$ even for unseen state-action pairs using function approximation, bypassing the requirement of a simulator. 
Specifically, observe that according to \pref{eq: dilated},
for each $x\in X_h$, under \pref{assum: linear MDP assumption} $\bonusQ_t(x,a)$ can be written as $\bonus_t(x,a) + \phi(x,a)^\top \Lambda_{t,h}^{\pi_t}$, where $\Lambda_{t,h}^{\pi_t} = (1+\frac{1}{H})\int_{x'\in X_{h+1}}\E_{a'\sim \pi_t(\cdot|x')} [\bonusQ_t(x',a')] \nu_h^{x'} \de x'$ is a vector independent of $(x,a)$. 
Thus, following the similar idea of using $\hattheta_{t,h}$ to estimate $\theta^{\pi_t}_{t,h}$ as we did in \pref{alg: linear Q}, we can construct $\bonusVec_{t,h}$ to estimate $\Lambda_{t,h}^{\pi_t}$ as well, thus succinctly representing $\bonusQ_t(x,a)$ for all $(x,a)$.

\paragraph{Epoch schedule}
Recall that estimating $\theta^{\pi_t}_{t,h}$ (and thus also $\Lambda_{t,h}^{\pi_t}$) requires constructing the covariance matrix inverse estimate $\hatcov_{t,h}$.
Due to the lack of a simulator, another important change of the algorithm is to construct $\hatcov_{t,h}$ using \emph{online} samples.
To do so, we divide the entire horizon (or more accurately the last $T-T_0$ rounds since the first $T_0$ rounds are reserved for some other purpose to be discussed next) into epochs with equal length $W$, and only update the policy optimization algorithm at the beginning of an epoch.
We index an epoch by $k$, and thus $\theta^{\pi_t}_{t,h}$, $\Lambda_{t,h}^{\pi_t}$, $\hatcov_{t,h}$ are now denoted by $\theta^{\pi_k}_{k,h}$, $\Lambda_{k,h}^{\pi_k}$, $\hatcov_{k,h}$. 
Within an epoch, we keep executing the same policy $\pi_k$ (up to a small exploration probability $\explore$) and collect $W$ trajectories, which are then used to construct $\hatcov_{k,h}$ as well as $\theta^{\pi_k}_{k,h}$ and $\Lambda_{k,h}^{\pi_k}$.
To decouple their dependence, we uniformly at random partition these $W$ trajectories into two sets $S$ and $S'$ with equal size, and use data from $S$ to construct $\hatcov_{k,h}$ in \textbf{Step~2} via the same \GR procedure and data from $S'$ to construct $\theta^{\pi_k}_{k,h}$ and $\Lambda_{k,h}^{\pi_k}$ in \textbf{Step~3} and \textbf{Step~4} respectively.

\paragraph{Exploration with a policy cover}
Unfortunately, some technical difficulty arises when bounding the estimation error and the variance of $\bonusVec_{k,h}$. Specifically, they can be large if the magnitude of the bonus term $\bonus_k(x,a)$ is large for some $(x,a)$; furthermore, since $\bonusVec_{k,h}$ is constructed using empirical samples, its variance can be even larger in those directions of the feature space that are rarely visited. Overall, due to the combined effect of these two facts, we are unable to prove any sublinear regret with only the ideas described so far.

To address this issue, we adopt the idea of \emph{policy cover}, recently introduced in \citep{agarwal2020pc, zanette2021cautiously}. 
Specifically, we spend the first $T_0$ rounds to find an exploratory (mixture) policy $\pcov$ (called policy cover) which tends to reach all possible directions of the feature space. 
This is done via the procedure $\findpcov$ (\pref{alg: find policy cover}) (to be discussed in detail soon), which also returns $\widehat{\cov}_h^{\textcov}$ for each layer $h$, an estimator of the true covariance matrix $\cov_h^{\textcov}$ of the policy cover $\pcov$.
$\findpcov$ guarantees that with high probability, for any policy $\pi$ and $h$ we have 
    \begin{align}
       \Pr_{x_h \sim \pi}\left[\exists a, \; \|\phi(x_h,a)\|_{(\widehat{\cov}_h^{\textcov})^{-1}}^2 \geq \alpha \right] \leq \otil\left(\frac{dH}{\alpha} \right) \label{eq:cover_guarantee}
    \end{align}
where $x_h \in X_h$ is sampled from executing $\pi$;
see \pref{lem: property of cover}.
This motivates us to only focus on $x$ such that $\|\phi(x,a)\|_{(\widehat{\cov}_h^{\textcov})^{-1}}^2 \leq \alpha$ for all $a$ ($h$ is the layer to which $x$ belongs).
This would not incur much regret because no policy would visit other states often enough.
We call such state a \textit{known} state and denote by $\calK$ the set of all known states.
To implement the idea above, we simply introduce an indicator $\ind[x \in \calK]$ in the definition of $\bonus_k$ (that is, no bonus at all for unknown states).

The benefit of doing so is that the aforementioned issue of $b_k(x,a)$ having a large magnitude is now alleviated as long as we explore using $\pcov$ with some small probability in each episode.
Specifically, in each episode of epoch $k$, with probability $1-\explore$ we execute $\pi_k$ suggested by policy optimization, otherwise we explore using $\pcov$.
The way we explore differs slightly for episodes in $S$ and those in $S'$ (recall that an epoch is partitioned evenly into $S$ and $S'$, where $S$ is used to estimate $\hatcov_{k,h}$ and $S'$ is used to estimate $\theta^{\pi_k}_{k,h}$ and $\Lambda_{k,h}^{\pi_k}$).
For an episode in $S$, we simply explore by executing $\pcov$ for the entire episode,
so that $\hatcov_{k,h}$ is an estimation of the inverse of $\gamma I + \explore \cov_h^{\textcov} + (1-\explore)\E_{(x_h,a)\sim \pi_k}[\phi(x_h,a)\phi(x_h,a)^\top]$,
and thus by its definition $b_k(x,a)$ is bounded by roughly $\frac{\alpha\beta}{\explore}$ for all $(x,a)$ (this improves over the trivial bound $\frac{\beta}{\gamma}$ by our choice of parameters; see \pref{lem: bonus bounded by 1}).
On the other hand, for an episode in $S'$, we first uniformly at random draw a step $h_t^*$, then we execute $\pcov$ for the first $h_t^*$ steps and continue with $\pi_k$ for the rest.
This leads to a slightly different form of the estimators $\theta^{\pi_k}_{k,h}$ and $\Lambda_{k,h}^{\pi_k}$ compared to \pref{eq: theta estimator in linear Q} (see  \pref{eq: theta estimator in linear MDP} and \pref{eq: bonus function estimator in linear MDP}, where the definition of $D_{t,h}$ is in light of \pref{eq: dilated}), which is important to ensure their (almost) unbiasedness.
This also concludes the description of \textbf{Step 1}.


We note that the idea of dividing states into known and unknown parts is
related to those of \citep{agarwal2020pc, zanette2021cautiously}. However, our case is more challenging because we are only allowed to mix a small amount of $\pcov$ into our policy in order to get sublinear regret against an adversary, while their algorithms can always start by executing $\pcov$ in each episode to maximally explore the feature space. 

\paragraph{Constructing the policy cover}
Finally, we describe how \pref{alg: find policy cover} finds a policy cover $\pcov$.
It is a procedure very similar to Algorithm~1 of \citep{wang2020reward}. Note that the focus of \cite{wang2020reward} is reward-free exploration in linear MDPs, but it turns out that the same idea can be used for our purpose, and it is also related to the exploration strategy introduced in \citep{agarwal2020pc, zanette2021cautiously}. 

More specifically, $\findpcov$ interacts with the environment for $T_0 = M_0 N_0$ rounds.
At the beginning of episode $(m-1)N_0+1$ for every $m=1, \ldots, M_0$,
it computes a policy $\pi_m$ using the LSVI-UCB algorithm of~\citep{jin2020provably} but with a fake reward function \pref{eq:fake_reward} (ignoring the true loss feedback from the environment).
This fake reward function is designed to encourage the learner to explore unseen state-action pairs and to ensure \pref{eq:cover_guarantee} eventually.
For this purpose, we could have set the fake reward for $(x,a)$ to be $\one\big[\|\phi(x,a)\|_{\Gamma_{m,h}^{-1}}^2 \geq \frac{\alpha}{2M_0}\big]$.
However, for technical reasons the analysis requires the reward function to be Lipschitz, and thus we approximate the indicator function above using a ramp function (with a large slope $T$).
With $\pi_m$ in hand, the algorithm then interacts with the environment for $N_0$ episodes, collecting trajectories to construct a good estimator of the covariance matrix of $\pi_m$.
The design of the fake reward function and this extra step of covariance estimation are the only differences compared to Algorithm~1 of~\citep{wang2020reward}. 
At the end of the procedure, $\findpcov$ construct $\pcov$ as a uniform mixture of $\{\pi_m\}_{m=1, \ldots, M_0}$.
This means that whenever we execute $\pcov$, we first uniformly at random sample $m\in [M_0]$, and then execute the (pure) policy $\pi_m$.


\begin{algorithm}[h]
    \caption{\findpcov$(M_0, N_0, \alpha, \delta)$}
    \label{alg: find policy cover}
    Let $\xi = 60d H\sqrt{\log(T/\delta)}$.  
    
    Let $\Gamma_{1,h}$ be the identity matrix in $\fR^{d\times d}$ for all $h$. 
    
    \For{$m=1,\ldots, M_0$}{ 
        Let $\Vht_{m}(x_H)= 0$. 
        
        \For{$h=H-1, H-2, \ldots, 0$}{
        For all $(x,a) \in X_h \times A$, compute
        \begin{align*}
        \Qht_{m}(x,a) &= \min\left\{r_m(x,a) + \xi\|\phi(x,a)\|_{\Gamma_{m,h}^{-1}} + \phi(x,a)^\top \hattheta_{m,h}, \ \ H\right\}, \\
                \Vht_m(x)&= \max_{a'} \Qht_m(x,a'), \\
                \pi_m(a|x) &= \one\left[a = \argmax_{a'} \Qht_m(x,a')\right], \tag{break tie in $\argmax$ arbitrarily} 
        \end{align*}
        with
            \begin{align}
                r_m(x,a)&= \ramp_{\frac{1}{T}}\left(\|\phi(x,a)\|_{\Gamma_{m,h}^{-1}}^2 - \frac{\alpha}{M_0}\right), \label{eq:fake_reward} \\ 
                \hattheta_{m,h}
                &= \Gamma_{m,h}^{-1}\left(\frac{1}{N_0}\sum_{t=1}^{(m-1)N_0}\phi(x_{t,h}, a_{t,h})  \Vht_m(x_{t, h+1})\right), \notag
        \end{align} 
        where 
        $\displaystyle
            \ramp_z(y)  = 
            \begin{cases}
                0  &\text{if\ } y\leq -z, \\
                1  &\text{if\ } y\geq 0, \\
                \frac{y}{z}+1 &\text{if\ } -z<y<0. 
            \end{cases}
        $
        }
        
        \For{$t=(m-1)N_0+1, \ldots, mN_0$}{
            Execute $\pi_m$ in episode $t$ and collect trajectory $\{x_{t,h}, a_{t,h}\}_{h=0}^{H-1}$. 
        }
        Compute
        \begin{align*}
            \Gamma_{m+1,h} = \Gamma_{m,h} +  \frac{1}{N_0}\sum_{t=(m-1)N_0+1}^{mN_0} \phi(x_{t,h}, a_{t,h})\phi(x_{t,h}, a_{t,h})^\top.  
        \end{align*}
    }
    Let $\pcov = \textsc{Uniform}\left(\{\pi_m\}_{m=1}^{M_0}\right)$ and $\widehat{\cov}^{\textcov}_h = \frac{1}{M_0}\Gamma_{M_0+1, h}$ for all $h$.  
    
    \textbf{return} $\pcov$ and $\left\{\widehat{\cov}^{\textcov}_h\right\}_{h=0, \ldots, H-1}$. 
\end{algorithm}

\paragraph{Regret guarantee}
With all these elements, 
we successfully remove the need of a simulator and prove the following regret guarantee.
\begin{theorem}\label{thm: linear MDP theorem}
    Under \pref{assum: linear MDP assumption}, \pref{alg:linearMDP} with appropriate choices of the parameters ensures $\E[\Reg]=\otil\left(d^2H^4T^{\nicefrac{14}{15}}\right)$.
\end{theorem}

\ifnobug
Although our regret rate is significantly worse than that in the full-information setting~\citep{cai2020provably}, in the stochastic setting~\citep{zanette2021cautiously}, or in the case when the transition is known~\citep{neu2021online},
we emphasize again that our algorithm is the first with provable sublinear regret guarantee for this challenging adversarial setting with bandit feedback and unknown transition.
\fi

\section{Conclusions and Future Directions}
In this work, we propose the general idea of dilated bonuses and demonstrate how it leads to improved exploration and regret bounds for policy optimization in various settings.
One future direction is to further improve our results in the function approximation setting, including reducing the number of simulator calls in the linear-$Q$ setting and improving the regret bound for the linear MDP setting (which is currently far from optimal).
A potential idea for the latter is to reuse data across different epochs, an idea adopted by several recent works~\citep{zanette2021cautiously, lazic2021improved} for different problems.
Another key future direction is to investigate whether the idea of dilated bonuses is applicable beyond the finite-horizon setting (e.g. whether it is applicable to the more general stochastic shortest path model or the infinite-horizon setting).

\paragraph{Acknowledgments}
We thank Gergely Neu and Julia Olkhovskaya for discussions on the technical details of their \GR procedure.

\bibliography{ref}
\bibliographystyle{plainnat}
\newpage
\appendix
\noindent{\LARGE \textbf{Appendix}}

\section{Auxiliary Lemmas}
In this section, we list auxiliary lemmas that are useful in our analysis.
First, we show some concentration inequalities. 
\begin{lemma}[(A special form of) Freedman's inequality, Theorem~1 of~\citep{beygelzimer2011contextual}]
\label{lem: freedman}
Let $\mathcal{F}_0 \subset \cdots \subset \mathcal{F}_n$ be a filtration, and $X_1, \ldots, X_n$ be real random variables such that $X_i$ is $\mathcal{F}_i$-measurable, $\E[X_i|\mathcal{F}_{i}]=0$, $|X_i|\leq b$, and $\sum_{i=1}^n \E[X_i^2|\mathcal{F}_{i}] \leq V$ for some fixed $b \geq 0$ and $V \geq 0$.
Then for any $\delta\in (0,1)$, we have with probability at least $1-\delta$,
\begin{align*}
    \sum_{i=1}^n   X_i \leq \frac{V}{b} + b\log(1/\delta).
\end{align*}
\end{lemma}

Throughout the appendix, we let $\calF_t$ be the $\sigma$-algebra generated by the observations before episode $t$.
\begin{lemma}[Adapted from Lemma 11 of \citep{jin2019learning}] \label{lem: useful 3}
    For all $x,a$, let $\{z_t(x,a)\}_{t=1}^T$ be a sequence of functions where $z_t(x,a)\in[0,R]$ is $\calF_t$-measurable. Let $Z_t(x,a)\in[0, R]$ be a random variable such that $\E_t[Z_t(x,a)]=z_t(x,a)$. Then with probability at least $1-\delta$,  
    \begin{align*}
        \sum_{t=1}^T \sum_{x, a} \left( \frac{\ind_t(x,a)Z_t(x,a)}{\Ut(x,a)+\gamma} - \frac{\QQ_t(x,a)z_t(x,a)}{\Ut(x,a)} \right) \leq \frac{RH}{2\gamma}\ln\frac{H}{\delta}. 
    \end{align*}
\end{lemma}

\begin{lemma}[Matrix Azuma, Theorem 7.1 of \citep{tropp2012user}] \label{lem: matrix azuma}
Consider an adapted sequence $\{X_k\}_{k=1}^n$ of self-adjoint matrices in dimension $d$, and a fixed sequence $\{A_k\}_{k=1}^n$ of self-adjoint matrices that satisfy 
\begin{align*}
    \E_{k}[X_k] = 0 \text{\ \ and \ \ } X_k^2 \preceq A_k^2 \text{\ almost surely}
\end{align*}

Define the variance parameter 
\begin{align*}
    \sigma^2 = \norm{\frac{1}{n}\sum_{k=1}^n A_k^2}_{\text{op}}. 
\end{align*}
Then, for all $\tau>0$, 
    \begin{align*}
        \Pr\left\{\norm{\frac{1}{n}\sum_{k=1}^n X_k}_{\text{op}} \geq \tau \right\} \leq de^{-n\tau^2/8\sigma^2}.  
    \end{align*}
\end{lemma}
Next, we show a classic regret bound for the exponential weight algorithm, which can be found, for example, in \citep{luo2017}.
\begin{lemma}[Regret bound of exponential weight, extracted from Theorem 1 of \citep{luo2017}]\label{lem: exponential weight lemma}
Let $\eta>0$, and let $\pi_t\in \Delta(A)$ and $\ell_t \in \mathbb{R}^A$ satisfy the following for all $t\in [T]$ and $a\in A$:   
\begin{align*}
    \pi_1(a) &= \frac{1}{|A|},  \\
    \pi_{t+1}(a) &= \frac{\pi_{t}(a)e^{-\eta \ell_t(a)}}{\sum_{a'\in A}\pi_{t}(a')e^{-\eta \ell_t(a')}}, \\
    |\eta \ell_{t}(a)| &\leq 1. 
\end{align*}
Then for any $\pistar\in \Delta(A)$, 
\begin{align*}
    \sum_{t=1}^T \sum_{a\in A} (\pi_t(a) - \pistar(a))\ell_t(a) \leq \frac{\ln |A|}{\eta} + \eta\sum_{t=1}^T \sum_{a\in A} \pi_t(a)\ell_{t}(a)^2.  
\end{align*}

\end{lemma}

\section{Proofs Omitted in \pref{sec: general recipe}}  
In this section, we prove \pref{lem:key}. In fact, we prove two generalized versions of it.
\pref{lem: lemma for the condition} states that the lemma holds even when we replace the definition of $\bonusQ_t(x,a)$ by an upper bound of the right hand side of \pref{eq: dilated}. (Note that \pref{lem:key} is clearly a special case with $\Pht = P$.)

\begin{lemma}\label{lem: lemma for the condition}
    Let $\bonus_t(x,a)$ be a non-negative loss function, and $\Pht$ be a transition function. Suppose that the following holds for all $x,a$: 
    \begin{align}
         \bonusQ_t(x,a) 
         &= \bonus_t(x,a) + \left(1+\frac{1}{H}\right)\E_{x'\sim \Pht(\cdot|x,a)}\E_{a'\sim \pi_t(\cdot|x')}\left[\bonusQ_t(x',a')\right] \label{eq: generalize dilate}\\
         &\geq \bonus_t(x,a) + \left(1+\frac{1}{H}\right)\E_{x'\sim P(\cdot|x,a)}\E_{a'\sim \pi_t(\cdot|x')}\left[\bonusQ_t(x',a')\right]   \nonumber 
    \end{align}
    with $\bonusQ_t(x_H,a)\triangleq 0$, and suppose that \pref{eq:modified bound}  holds. Then  
    \begin{align*}
        \Reg \leq o(T) + 3\sum_{t=1}^T\Vht^{\pi_t}(x_0; \bonus_t).  
    \end{align*}
    where $\Vht^{\pi}$ is the state value function under the transition function $\Pht$ and policy $\pi$. 
\end{lemma}

\begin{proof}[Proof of \pref{lem: lemma for the condition}]
    By rearranging \pref{eq:modified bound}, we see that  
    \begin{align*}
    \Reg &\leq  o(T) + 
    \underbrace{\sum_{t=1}^T \sum_{x,a} \qstar(x)\pistar(a|x)\bonus_t(x,a)}_{\term_1} \\
    &\qquad \qquad + \underbrace{\frac{1}{H}\sum_{t=1}^T \sum_{x,a} \qstar(x)\pi_t(a|x)\bonusQ_t(x,a)}_{\term_2} + \underbrace{\sum_{t=1}^T \sum_{x,a} \qstar(x)\Big(\pi_t(a|x) - \pistar(a|x)\Big) \bonusQ_t(x,a)}_{\term_3}.  
    \end{align*}
    
    We first focus on $\term_3$, and focus on a single layer $0\leq h\leq H-1$ and a single $t$:  
    \begin{align*}
        &\sum_{x\in X_h}\sum_{a\in A}\qstar(x)\left(\pi_t(a|x) - \pistar(a|x)\right)\bonusQ_t(x,a)\\
        &= \sum_{x\in X_h}\sum_{a\in A}\qstar(x)\pi_t(a|x)\bonusQ_t(x,a) - \sum_{x\in X_h}\sum_{a\in A} \qstar(x)\pistar(a|x)\bonusQ_t(x,a)\\
        &= \sum_{x\in X_h}\sum_{a\in A}\qstar(x)\pi_t(a|x)\bonusQ_t(x,a) \\
        &\qquad \qquad - \sum_{x\in X_h}\sum_{a\in A} \qstar(x)\pistar(a|x)\left(\bonus_t(x,a) + \left(1+\frac{1}{H}\right)\E_{x'\sim \Pht(\cdot|x,a)}\E_{a'\sim \pi_t(\cdot|x')}\left[  \bonusQ_t(x',a')\right]\right)\\
        &\leq \sum_{x\in X_h}\sum_{a\in  A}\qstar(x)\pi_t(a|x)\bonusQ_t(x,a) \\
        &\qquad \qquad - \sum_{x\in X_h}\sum_{a\in A} \qstar(x)\pistar(a|x)\left(\bonus_t(x,a) + \left(1+\frac{1}{H}\right)\E_{x'\sim P(\cdot|x,a)}\E_{a'\sim \pi_t(\cdot|x')}\left[  \bonusQ_t(x',a')\right]\right)\\
        &= \sum_{x\in X_h}\sum_{a\in A}\qstar(x)\pi_t(a|x)\bonusQ_t(x,a) - \sum_{x\in X_{h+1}}\sum_{a\in A} \qstar(x)\pi_t(a|x)\bonusQ_t(x,a) \\
        &\qquad \qquad - \sum_{x\in X_{h}}\sum_{a\in A} \qstar(x)\pistar(a|x)\bonus_t(x,a) - \frac{1}{H}\sum_{x\in X_{h+1}}\sum_{a\in A} \qstar(x)\pi_t(a|x)\bonusQ_t(x,a),
    \end{align*}
    where the last step uses the fact $\sum_{x\in X_h}\sum_{a\in A} \qstar(x)\pistar(a|x)P(x'|x,a) = \qstar(x')$ (and then changes the notation $(x',a')$ to $(x,a)$).
    Now summing this over $h=0, 1, \ldots, H-1$ and $t=1,\ldots, T$, and combining with $\term_1$ and $\term_2$, we get 
    \begin{align*}
        \term_1 + \term_2 + \term_3 
        &= \left(1+\frac{1}{H}\right)\sum_{t=1}^T \sum_a\pi_t(a|x_0)\bonusQ_t(x_0,a).
    \end{align*}
    Finally, we relate $\sum_a\pi_t(a|x_0)\bonusQ_t(x_0,a)$ to $\Vht^{\pi_t}(x_0; \bonus_t)$. 
    Below, we show by induction that for $x\in X_h$ and any $a$, 
    \begin{align*}
        \sum_{a\in A}\pi_t(a|x)\bonusQ_t(x,a) \leq \left(1+\frac{1}{H}\right)^{H-h-1}\Vht^{\pi_t}(x; \bonus_t).
    \end{align*}
    When $h=H-1$, $\sum_{a}\pi_t(a|x)\bonusQ_t(x,a)=\sum_{a}\pi_t(a|x)\bonus_t(x,a) = \Vht^{\pi_t}(x;\bonus_t)$. Suppose that the hypothesis holds for all $x\in X_h$. Then for any $x\in X_{h-1}$, \begin{align*}
        \sum_{a\in A}\pi_t(a|x)\bonusQ_t(x,a)
        &= \sum_a \pi_t(a|x)\Big(\bonus_t(x,a) + \left(1+\frac{1}{H}\right)\E_{x'\sim \Pht(\cdot|x,a)}\E_{a'\sim \pi_t(\cdot|x')}\left[B_t(x',a')\right]\Big) \\
        &\leq \sum_a \pi_t(a|x)\Big(\bonus_t(x,a) + \left(1+\frac{1}{H}\right)^{H-h}\E_{x'\sim \Pht(\cdot|x,a)}\left[\Vht^{\pi_t}(x';\bonus_t)\right]\Big)   \tag{induction hypothesis}\\
        &\leq \left(1+\frac{1}{H}\right)^{H-h}\sum_a \pi_t(a|x)\Big(\bonus_t(x,a) + \E_{x'\sim \Pht(\cdot|x,a)}\left[\Vht^{\pi_t}(x';\bonus_t)\right]\Big) \tag{$\bonus_t(x,a)\geq 0$}\\
        &= \left(1+\frac{1}{H}\right)^{H-h} \Vht^{\pi_t}(x; \bonus_t),  
    \end{align*} 
    finishing the induction. 
    Applying the relation on $x=x_0$ and noticing that $\left(1+\frac{1}{H}\right)^{H}\leq e<3$ finishes the proof. 
\end{proof}
Besides \pref{lem: lemma for the condition}, we also show \pref{lem: expected version} below, which guarantees that \pref{lem:key} holds even if \pref{eq: dilated} and \pref{eq:modified bound} only hold in expectation. 
\begin{lemma}\label{lem: expected version}
     Let $\bonus_t(x,a)$ be a non-negative loss function that is fixed at the beginning of episode $t$, and let $\pi_t$ be fixed at the beginning of episode $t$. Let $\bonusQ_t(x,a)$ be a randomized bonus function that satisfies the following for all $x,a$: 
    \begin{align}
         \E_t\left[\bonusQ_t(x,a)\right] 
         &= \bonus_t(x,a) + \left(1+\frac{1}{H}\right)\E_{x'\sim P(\cdot|x,a)}\E_{a'\sim \pi_t(\cdot|x')}\E_t\Big[\bonusQ_t(x',a')\Big]   \label{eq: expected condition 2}
    \end{align}
    with $\bonusQ_t(x_H,a)\triangleq 0$, and suppose that the following holds (simply taking expectations on \pref{eq:modified bound}): 
    \begin{align}
    &\E\left[\sum_{x} \qstar(x)\sum_{t=1}^T  \sum_a \Big(\pi_t(a|x) - \pistar(a|x)\Big)\Big(Q^{\pi_t}_t(x,a) - \bonusQ_t(x,a) \Big) \nonumber\right] \\
    &\qquad \qquad \leq o(T) + 
    \E\left[\sum_{t=1}^T V^{\pistar}(x_0; \bonus_t)\right] + \frac{1}{H}\E\left[\sum_{t=1}^T \sum_{x,a} \qstar(x)\pi_t(a|x)\bonusQ_t(x,a)\right]. \label{eq:modified bound expected} 
\end{align}
    
    Then  
    \begin{align*}
        \E\left[\Reg\right] \leq o(T) + 3\E\left[\sum_{t=1}^T V^{\pi_t}(x_0; \bonus_t)\right].  
    \end{align*}
\end{lemma}

\begin{proof}
     The proof of this lemma follows that of \pref{lem: lemma for the condition} line-by-line (with $\Pht = P$), except that we take expectations in all steps. 
\end{proof}

\section{Details Omitted in \pref{sec: tabular}} \label{app: tabular appendix} 
In this section, we first discuss the implementation details of \pref{alg:MDP-traj} in \pref{sec:comp uob lob}, then we give the complete proof of \pref{thm:regret bound main} in \pref{sec:proofs sec4}.

\subsection{Implementation Details}\label{sec:comp uob lob}
\label{app: omitted procedure}
The \compuob procedure  is the same as Algorithm 3 of \citep{jin2019learning}, which shows how to efficiently compute an upper occupancy bound. We include the algorithm in \pref{alg:uob} for completeness. As \pref{alg:MDP-traj} also needs \complob, which computes a lower occupancy bound, we provide its complete pseudocode in \pref{alg:lob} as well.

Fix a state $x$. Define $f(\xtilde)$ to be the maximum and minimum probability of visiting $x$ starting from state $\xtilde$ for \compuob and \complob, respectively. Then the two algorithms almost have the same procedure to find $f(\xtilde)$ by solving the optimization in \pref{eq:DP} subject to $\Phat$ in the confidence set $\mathcal{P}$ via a greedy approach in \pref{alg:greedy}.
The difference is that \compuob sets $\obj$ to be $\max$ while \complob sets $\obj$ to be $\min$, and thus in \pref{alg:greedy}, $\{f(x)\}_{x\in X_k}$ is sorted in an ascending and a descending order, respectively.

Finally, we point out that the bonus function $\bonusQ_t(s,a)$ defined in \pref{eq: bonus bellman 1} can clearly also be computed using a greedy procedure similar to \pref{alg:greedy}.
This concludes that the entire algorithm can be implemented efficiently.

\begin{equation}\label{eq:DP}
f(\xtilde) = \sum_{a\in A}\pi_t(a|\xtilde)\left( \underset{\Phat(\cdot|\xtilde,a)}{\obj} \sum_{x' \in X_{k(\xtilde)+1}} \Phat(x' | \xtilde, a) f(x') \right)
\end{equation}

\begin{algorithm}[H]
    \caption{\compuob (Algorithm 3 of \citep{jin2019learning})}\label{alg:uob} 
{\bfseries Input:} a policy $\pi_t$, a state-action pair $(x,a)$ and a confidence set $\mathcal{P}$ of the form

\[
\left\{\Phat: \left|\Phat(x'|x,a) - \bar{P}(x'|x,a)\right|\leq \epsilon(x'|x,a), \;\forall(x,a,x'
) \right\}
\]

{\bfseries Initialize:} for all $\xtilde \in X_{k(x)}$, set $f(\xtilde) = \ind\{\xtilde=x\}$.

\For{ $k = k(x)-1 \ \textbf{to} \ 0$}{

    \For{ $\forall\xtilde \in X_k$}{
    
    Compute $f(\xtilde)$ based on :
    
    \begin{align*}
    f(\xtilde) = \sum_{a\in A} \pi_t(a|\xtilde) \cdot \text{\textsc{Greedy}}\left(f,
    \bar{P}(\cdot|\xtilde,a), \epsilon(\cdot|\xtilde,a),\max\right)
    \end{align*}
    
    }
}

{\bfseries Return:} $\pi_t(a|x) f(x_0)$.
\end{algorithm}

\begin{algorithm}[H]
    \caption{\complob}
    {\bfseries Input:}\label{alg:lob} a policy $\pi_t$, a state-action pair $(x,a)$ and a confidence set $\mathcal{P}$ of the form

\[
\left\{\Phat: \left|\Phat(x'|x,a) - \bar{P}(x'|x,a)\right|\leq \epsilon(x'|x,a), \;\forall(x,a,x'
) \right\}
\]

{\bfseries Initialize:} for all $\xtilde \in X_{k(x)}$, set $f(\xtilde) = \ind\{\xtilde=x\}$.

\For{ $k = k(x)-1 \ \textbf{to} \ 0$}{

    \For{ $\forall\xtilde \in X_k$}{
    
    Compute $f(\xtilde)$ based on :
    
    \begin{align*}
    f(\xtilde) = \sum_{a\in A} \pi_t(a|\xtilde) \cdot \text{\textsc{Greedy}}\left(f,
    \bar{P}(\cdot|\xtilde,a), \epsilon(\cdot|\xtilde,a),\min\right)
    \end{align*}
    
    }
}

{\bfseries Return:} $\pi_t(a|x) f(x_0)$.
\end{algorithm}

\begin{algorithm}[!htbp]
\caption{\textsc{Greedy}}
\label{alg:greedy}

{\bfseries Input:} $f: X\rightarrow[0,1]$, a distribution $\bar{p}$ over $n$ states of layer $k$ , positive numbers $\{\epsilon(x)\}_{x \in X_k}$, objective \obj ($\max$ for \compuob and $\min$ for \complob). 

{\bfseries Initialize:} $j^- = 1, j^+ = n$, sort $\{f(x)\}_{x\in X_k}$ and find $\sigma$ such that
\[
f(\sigma(1)) \leq f(\sigma(2)) \leq \cdots \leq f(\sigma(n))
\]
for $\obj=\max$, and 
\[
f(\sigma(1)) \geq f(\sigma(2)) \geq \cdots \geq f(\sigma(n))
\]
for $\obj=\min$. 

\While{$j^- < j^+$}{
 $x^- =\sigma(j^-), x^+ =\sigma(j^+)$
 
$\delta^{-} = \min\{\bar{p}(x^-),  \epsilon(x^-)\}$ 

$\delta^{+} = \min\{1 - \bar{p}(x^+),  \epsilon(x^+)\}$ 

$\bar{p}(x^-) \leftarrow \bar{p}(x^-) - \min\{\delta^{-}, \delta^{+}\}$

$\bar{p}(x^+) \leftarrow \bar{p}(x^+) + \min\{\delta^{-}, \delta^{+}\}$

\If{$\delta_{-} \leq \delta_{+}$}{

$\epsilon(x^+) \leftarrow \epsilon(x^+) - \delta^{-}$

$j^- \leftarrow j^- + 1$
}
\Else{
$\epsilon(x^-) \leftarrow \epsilon(x^-) - \delta^{+}$

$j^+ \leftarrow j^+ - 1$}
}

{\bfseries Return:} $\sum_{j=1}^n \bar{p}(\sigma(j))f(\sigma(j))$

\end{algorithm}

\subsection{Omitted Proofs}\label{sec:proofs sec4}

To prove \pref{thm:regret bound main},
as discussed in the analysis sketch of \pref{sec: tabular},
we decompose the left-hand side of \pref{eq:modified bound} as:
\begin{align}
	&\sum_{t=1}^T\sum_x\qstar(x)\inner{\pi_t(\cdot|x)-\pistar(\cdot|x), \Q_t^{\pi_t}(x,\cdot) -B_t(x,\cdot)} \nonumber\\
	&= \underbrace{\sum_{t=1}^T\sum_x\qstar(x)\inner{\pi_t(\cdot|x), \Q_t^{\pi_t}(x,\cdot)-\Qht_t(x,\cdot)}}_{\bias}+\underbrace{\sum_{t=1}^T\sum_x\qstar(x)\inner{\pistar(\cdot|x), \Qht_t(x,\cdot)-\Q_t^{\pi_t}(x,\cdot)}}_{\biastwo}\nonumber\\
	&\quad
	 + \underbrace{\sum_{t=1}^T\sum_x\qstar(x)\inner{\pi_t(\cdot|x)-\pistar(\cdot|x), \Qht_t(x,\cdot)-B_t(x,\cdot)}}_{\regterm}\label{eq:decompose tabular regret}.
\end{align}
We bound each term in a corresponding lemma. Specifically, We show a high probability bound of $\bias$ in \pref{lem: bias1}, a high probability bound of $\biastwo $ in \pref{lem: bias2}, and a high-probability bound of $\regterm$ in \pref{lem: reg}.
Finally, we show how to combine all terms with the definition of $\bonus_t$ in \pref{thm:regret bound}, which is a restatement of \pref{thm:regret bound main}.

\begin{lemma}[$\bias$]\label{lem: bias1}
	With probability at least $1-5\delta$, 
	\begin{align*}
	\bias &\le \otil\left(\frac{H}{\eta}\right)+\sum_{t=1}^T\ \sum_{x,a}\qstar(x)\pi_{t}(a| x)\left(\frac{2\gamma H+ H\left(\Ut(x, a)-\Lt(x, a)\right)}{\Ut(x, a)+\gamma}\right). 
\end{align*}

\end{lemma}

\begin{proof}
    In the proof, we assume that $P\in \calP_k$ for all $k$, with holds with probability at least $1-4\delta$ as already shown in \citep[Lemma 2]{jin2019learning}. Under this event, $\Lt(x,a)\leq \QQ_t(x,a)\leq \Ut(x,a)$ for all $t, x, a$. 
    
    Let $Y_{t}=\sum_{x\in X}\qstar(x)\inner{\pi_t(\cdot|x),\Qht_t(x,\cdot)}$.
    First, we decompose \bias as
    \begin{align}
        \sum_{t=1}^T\left(\E_t[Y_t]-Y_t\right)+\left(\sum_x\qstar(x)\inner{\pi_t(\cdot|x), \Q_t^{\pi_t}(x,\cdot)}-\E_t[Y_t]\right).  \label{eq: bias 1 decompose}
    \end{align}
    We will bound the first Martingale sequence using Freedman's inequality.
    Note that we have
    \begin{align}
        \Var_t[Y_{t}]&\le\E_t\left[\left(\sum_{x}\qstar(x)\inner{\pi_t(\cdot|x), \Qht_t(x,\cdot)}\right)^2\right]   \nonumber \\
        &\le \E_t\left[\left(\sum_{x,a}\qstar(x)\pi_t(a|x)\right)\left(\sum_{x,a}\qstar(x){\pi_t(a|x)\Qht_t(x,a)^2}\right)\right]  \tag{Cauchy-Schwarz}\\
        &= H\sum_{x,a}\qstar(x)\pi_t(a|x)\frac{L_{t,h}^2\E_t[\ind_t(x,a)]}{(\Ut(x,a)+\gamma)^2}\nonumber   \tag{$\sum_{x,a}\qstar(x)\pi_t(a|x) = H$}\\
        &\le H\sum_{x,a}\qstar(x)\pi_t(a|x)\frac{\QQ_t(x,a)H^2}{(\Ut(x,a)+\gamma)^2}\nonumber   \tag{$L_{t,h} \leq H$ and $\E_t[\ind_t(x,a)] = \QQ_t(s,a)$}\\
        &\le \sum_{x,a}\qstar(x)\pi_t(a|x)\frac{H^3}{\Ut(x,a)+\gamma} \nonumber \tag{$\QQ_t(s,a) \leq \Ut(x,a)$}
    \end{align}
    and $|Y_t|\leq H\sup_{x,a}|\Qht(x,a)| \leq \frac{H^2}{\gamma}$. 
    
    Moreover, for every $t$, the second term in \pref{eq: bias 1 decompose} can be bounded as 
    \begin{align*}
        &\sum_x\qstar(x)\inner{\pi_t(\cdot|x), \Q_t^{\pi_t}(x,\cdot)}-\E_t\left[\sum_x\qstar(x)\inner{\pi_t(\cdot|x),\Qht_t(x,\cdot)}\right]\\
        &=\sum_{x,a}\qstar(x){\pi_t(a|x) \Q_t^{\pi_t}(x,a)\left(1-\frac{\QQ_t(x,a)}{\Ut(x,a)+\gamma}\right)}\\
        &\le \sum_{x,a}\qstar(x){\pi_t(a|x) H\left(\frac{\Ut(x,a)-\QQ_t(x,a)+\gamma}{\Ut(x,a)+\gamma}\right)} \tag{$\Q_t(x,a) \leq H$}\\
        &\le \sum_{x,a}\qstar(x){\pi_t(a|x) H\left(\frac{\Ut(x,a)-\Lt(x,a)+\gamma}{\Ut(x,a)+\gamma}\right)}. \tag{$\Lt(x,a) \leq \QQ_t(x,a)$}
    \end{align*}
    
    Combining them, and using Freedman's inequality (\pref{lem: freedman}), we have that with probability at least $1-5\delta$, 
    \begin{align*}
	\bias&= \sum_{t=1}^T\sum_x\qstar(x)\inner{\pi_t(\cdot|x), \Q_t^{\pi_t}(x,\cdot)-\Qht_t(x,\cdot)}\\
	&\le\sum_{t=1}^T\ \sum_{x,a}\qstar(x)\pi_{t}(a| x)H\left(\frac{\left(\Ut(x, a)-\Lt(x, a)\right)+\gamma}{\Ut(x, a)+\gamma}\right)\\
	&\qquad +  \frac{\gamma}{H^2}\sum_{t=1}^T\sum_{x,a}\qstar(x)\pi_t(a|x)\frac{H^3}{\Ut(x,a)+\gamma} + \frac{H^2}{\gamma}\ln \frac{1}{\delta} \\
	&\le \otil\left(\frac{H}{\eta}\right)+\sum_{t=1}^T\ \sum_{x,a}\qstar(x)\pi_{t}(a| x)\left(\frac{2\gamma H+ H\left(\Ut(x, a)-\Lt(x, a)\right)}{\Ut(x, a)+\gamma}\right),  
\end{align*}
where we use $\gamma=2\eta H$. 
\end{proof}
Next, we bound $\biastwo$.
\begin{lemma}[$\biastwo$]\label{lem: bias2}
	With probability at least $1-5\delta$, 
	$\biastwo\leq \otil\left(\frac{H}{\eta}\right)$. 
\end{lemma}

\begin{proof}
   We invoke \pref{lem: useful 3} with $z_t(x,a)=\qstar(x)\pistar(a|x)Q_t^{\pi_t}(x,a)$ and \[Z_t(x,a)=\qstar(x)\pistar(a|x)\left(\ind_t(x,a)L_t(x,a)+(1-\ind_t(x,a))Q_t^{\pi_t}(x,a)\right).\] Then we get that with probability at least $1-\delta$ (recalling the definition $\Qht_t(x,a) = \frac{L_{t,h}}{\Ut(x,a)+\gamma}\ind_t(x,a)$), 
    \begin{align}
        \sum_{t=1}^T \sum_{x, a} \qstar(x)\pistar(a|x)\left(\Qht_t(x,a) - \frac{\QQ_{t}(x,a)}{\Ut(x,a)}Q_t^{\pi_t}(x,a)\right)\leq  \frac{H^2}{2\gamma}\ln\frac{H}{\delta}, \label{eq: bias 2 bound 1} 
    \end{align} 
    
    Since with probability at least $1-4\delta$, $\QQ_{t}(x,a)\leq \Ut(x,a)$ for all $t, x, a$ (by \citep[Lemma 2]{jin2019learning}), \pref{eq: bias 2 bound 1} further implies that with probability at least $1-5\delta$,
    \begin{align*}
        \biastwo = \sum_{t=1}^T\sum_{x,a}\qstar(x)\pistar(x,a)\left(\Qht_t(x,a) - Q_t^{\pi_t}(x,a)\right)\leq \frac{H^2}{2\gamma}\ln\frac{H}{\delta}. 
    \end{align*}
    Noting that $\gamma=2\eta H$ finishes the proof. 
\end{proof}

We continue to bound $\regterm$.
\begin{lemma}[$\regterm$]\label{lem: reg}
With probability at least $1-5\delta$, 
\begin{align*}
	 \regterm
	 &\leq \otil\left(\frac{H}{\eta}\right) + \sum_{t=1}^T \sum_{x,a}\qstar(x)\pi_t(a|x) \left(\frac{\gamma H}{\Ut(x,a)+\gamma}+\frac{\bonusQ_t(x,a)}{H}\right). 
\end{align*}
\end{lemma}

\begin{proof}
    The algorithm runs individual exponential weight updates on each state with loss vectors $\Qht_t(x,\cdot)-B_t(x,\cdot)$, so we can apply standard results for exponential weight updates. Specifically, we can apply \pref{lem: exponential weight lemma} on each state $x$, and get 
    \begin{align}
        \sum_{t=1}^T \inner{\pi_t(\cdot|x)-\pistar(\cdot|x), \Qht_t(x,\cdot)-B_t(x,\cdot)}
        \leq \frac{\ln |A|}{\eta} + \eta\sum_{t=1}^T \sum_{a\in A}\pi_t(a|x)\left(\Qht_t(x,a)-B_t(x,a)\right)^2.
        \label{eq: reg term bound 1}
    \end{align}
    The condition required by \pref{lem: exponential weight lemma} (i.e., $\eta|\Qht_t(x,a)-B_t(x,a)|\leq 1$) is verified in \pref{lem: simple bound}.  Summing \pref{eq: reg term bound 1} over states with weights $\qstar(x)$, we get 
    \begin{align}
        \regterm 
        &\leq \frac{H\ln |A|}{\eta} + \eta\sum_{t=1}^T \sum_{x,a}\qstar(x)\pi_t(a|x)\left(\Qht_t(x,a)-B_t(x,a)\right)^2  \nonumber \\
        &\leq \frac{H\ln |A|}{\eta} + 2\eta\sum_{t=1}^T \sum_{x,a}\qstar(x)\pi_t(a|x)\Qht_t(x,a)^2 + 2\eta\sum_{t=1}^T \sum_{x,a}\qstar(x)\pi_t(a|x)B_t(x,a)^2. \label{eq: reg term bound 4}
    \end{align}
    
    Below, we focus on the last two terms on the right-hand side of \pref{eq: reg term bound 4}. First, we have
    \begin{align*}
        2\eta\sum_{t=1}^T \sum_{x,a}\qstar(x)\pi_t(a|x)\Qht_t(x,a)^2
        &\leq 2\eta\sum_{t=1}^T \sum_{x,a}\qstar(x)\pi_t(a|x)\frac{H^2\ind_t(x,a)}{(\Ut(x,a)+\gamma)^2}\\
        &= 2\eta H^2\sum_{t=1}^T \sum_{x,a}\frac{\qstar(x)\pi_t(a|x)}{\Ut(x,a)+\gamma}\cdot \frac{\ind_t(x,a)}{\Ut(x,a)+\gamma}\\
        &\le 2\eta H^2\sum_{t=1}^T \sum_{x,a}\frac{\qstar(x)\pi_t(a|x)}{\Ut(x,a)+\gamma}\cdot\frac{\QQ_t(x,a)}{\Ut(x,a)} + 2\eta H^2 \times \frac{\frac{H}{\gamma} \ln\frac{H}{\delta}}{2\gamma} \\
        &\le \frac{H}{4\eta}\ln\frac{H}{\delta}+ \sum_{t=1}^T \sum_{x,a}\qstar(x)\pi_t(a|x)\frac{\gamma H}{\Ut(x,a)+\gamma},
    \end{align*} 
    where the third step happens with probability at least $1-\delta$ by \pref{lem: useful 3} with $z_t(x,a) = Z_t(x,a) = \frac{\qstar(x)\pi_t(a|x)}{\Ut(x,a)+\gamma} \leq \frac{1}{\gamma}$, and the last step uses $\gamma=2\eta H$ and $\QQ_t(x,a)\leq \Ut(x,a)$ (which happens with probability at least $1-4\delta$).
    For the second term in \pref{eq: reg term bound 4}, note that
    \begin{align*}
        2\eta\sum_{t=1}^T \sum_{a\in A}\pi_t(a|x)B_t(x,a)^2
        &\leq \frac{1}{H}\sum_{t=1}^T \sum_{a\in A}\pi_t(a|x) B_t(x,a)
    \end{align*}
    due to the fact $\eta B_t(x,a)\leq \frac{1}{2H}$ by \pref{lem: simple bound}. 
    Combining everything finishes the proof. 
\end{proof}

In \pref{lem: reg}, as required by \pref{lem: exponential weight lemma}, we control the magnitude of $\eta \Qht_t(x,a)$ and $\eta \bonusQ_t(x,a)$ by setting $\gamma$ and $\eta$ properly, shown in the following technical lemma. 

\begin{lemma}
\label{lem: simple bound}
   
    $\eta \Qht_t(x,a)\leq \frac{1}{2}$ and $\eta \bonusQ_t(x,a)\leq \frac{1}{2H}$. 
\end{lemma}
\begin{proof}
     Recall that $\gamma=2\eta H$ and $\eta\leq \frac{1}{24 H^3}$. Thus,
    \begin{align*}
        \eta \Qht_t(x,a) &\leq \frac{\eta H}{\gamma} = \frac{\eta H}{2\eta H} = \frac{1}{2}, \\
        \eta \bonus_t(x,a) &=\frac{ 3\eta\gamma H + \eta H(\Ut(x,a)-\Lt(x,a)) }{\Ut(x, a) + \gamma} 
        \leq 3\eta H + \eta H \leq \frac{1}{6H^2}.
    \end{align*}
    By the definition of $\bonusQ_t(x,a)$ in \pref{eq: bonus bellman 1}, we have 
    \begin{align*}
        \eta \bonusQ_t(x,a) \leq  H\left(1+\frac{1}{H}\right)^H \eta\sup_{x',a'}\bonus_t(x',a')\leq 3H\times \frac{1}{6H^2} = \frac{1}{2H}.
    \end{align*} 
    This finishes the proof.
\end{proof}

Now we are ready to prove \pref{thm:regret bound main}. For convenience, we state the theorem again here and show the proof.

\begin{theorem}\label{thm:regret bound}
	\pref{alg:MDP-traj} ensures that with probability $1-\order(\delta)$, $
	\Reg = \otil\left(|X|H^2\sqrt{AT} + H^4\right)$. 
\end{theorem}
\begin{proof}
Combining \bias, \biastwo, \regterm, we get that with probability at least $1-\order(\delta)$, 
\begin{align*}
    &\bias+\biastwo+ \regterm \\
    &\le \otil\left(\frac{H}{\eta} \right)+  \sum_{t=1}^T \sum_{x,a}\qstar(x)\pi_t(a|x)\left(\frac{3\gamma H+H(\Ut(x,a)-\Lt(x,a)) }{\Ut(x,a)+\gamma} + \frac{1}{H}\bonusQ_t(x,a)\right) \\
    &= \otil\left(\frac{H}{\eta} \right)+  \sum_{t=1}^T \sum_{x,a}\qstar(x)\pistar(a|x)\bonus_t(x,a) + \frac{1}{H}\sum_{t=1}^T \sum_{x,a}\qstar(x)\pi_t(a|x)\bonusQ_t(x,a),
\end{align*}
which is of the form specified in \pref{eq:modified bound}. By the definition of $\bonusQ_t(x,a)$ in \pref{eq: bonus bellman 1}, we see that \pref{eq: generalize dilate} also holds with probability at least $1-\order(\delta)$ for all $t, x, a$.  
 
Therefore, by \pref{lem: lemma for the condition}, we can bound the regret as (let $\Pht_t$ be the optimistic transition function chosen in \pref{eq: bonus bellman 1} at episode $t$) 
\begin{align*}
    \Reg 
    &= \otil\left(\frac{H}{\eta} + \sum_{t=1}^T \sum_{x,a}\QQ^{\Pht_t, \pi_t}(x,a)\bonus_t(x,a)\right) \\
    &= \otil\left(\frac{H}{\eta} + \sum_{t=1}^T \sum_{x,a}\QQ^{\Pht_t, \pi_t}(x,a)\frac{H(\Ut(x,a)-\Lt(x,a))+\gamma H}{\Ut(x,a)+\gamma}\right)\\
    &= \otil\left(\frac{H}{\eta} + \sum_{t=1}^T \sum_{x,a}\left(H(\Ut(x,a)-\Lt(x,a))+\eta H^2\right)\right)  \tag{$\QQ^{\Pht_t, \pi_t}(x,a) \leq \Ut(x,a)$ and $\gamma=2\eta H$} \\
    &\le \otil\left(\frac{H}{\eta} + |X|H^2\sqrt{AT}+ \eta|X||A|H^2T\right),
\end{align*}
where the last inequality is due to \citep[Lemma 4]{jin2019learning}.
Plugging in the specified value for $\eta$, the regret can be further upper bounded by $  \otil\left(|X|H^2\sqrt{AT}+H^4\right)$. 
\end{proof}

\section{Analysis for Auxiliary Procedures}

In this section, we analyze two important auxiliary procedures for the linear function approximation settings: \GR and \findpcov.

\subsection{The Guarantee of \GR}
\label{app: GR analysis}

The \GR algorithm is shown in  \pref{alg: GR},
which is almost the same as that in \cite{neu2020online} except that we repeat the same procedure for $M$ times and average the outputs (see the extra outer loop).
This extra step is added to deal with some technical difficulties in the analysis.
The following lemma summarizes some useful guarantees of this procedure.
For generality, we present the lemma assuming a lower bound on the minimum eigenvalue $\lambda$ of the covariance matrix, but it will simply be $0$ in all our applications of this lemma in this work.

\begin{lemma}
    \label{lem: GR lemma contingency}
     Let $\pi$ be a policy (possibly a mixture policy) with a covariance matrix $\cov_{h} = \E_{\pi}[\phi(x_h, a_h)\phi(x_h, a_h)^\top]\succeq \lambda I$ for layer $h$ and some constant $\lambda\geq 0$. Further let $\epsilon>0$ and $\gamma\geq 0$ be two parameters satisfying $0< \gamma+\lambda < 1$.  Define $M=\left\lceil\frac{24\ln(dHT)}{\epsilon^2}\min\left\{\frac{1}{\gamma^2}, \frac{4}{\lambda^2}\ln^2\frac{1}{\epsilon\lambda}\right\}\right\rceil$ and $N=\left\lceil\frac{2}{\gamma+\lambda}\ln \frac{1}{\epsilon (\gamma+\lambda)} \right\rceil$. Let $\calT$ be a set of $MN$ trajectories generated by $\pi$. Then $\GR$ (\pref{alg: GR}) with input $(\calT,M,N,\gamma)$ ensures the following for all $h$: 
    \begin{align}  
        \norm{\hatcov_{h}}_{\text{\rm op}} &\leq \min\left\{\frac{1}{\gamma}, \frac{2}{\lambda}\ln\frac{1}{\epsilon\lambda} \right\}. \label{eq: GE 3} \\ 
        \norm{\E\left[\hatcov_{h}\right] - \left(\gamma I + \cov_{h}\right)^{-1}}_{\text{\rm op}} &\leq \epsilon, \label{eq: GE 5}\\
        \norm{\hatcov_{h} - \left(\gamma I + \cov_{h}\right)^{-1}}_{\text{\rm op}} &\leq 2\epsilon, \label{eq: GE 4} \\
        \norm{\hatcov_{h}\cov_{h}}_{\text{\rm op}}
        &\leq 1 + 2\epsilon,  \label{eq: bounded norm prod}
    \end{align}
    where $\norm{\cdot}_{\text{\rm op}}$ represents the spectral norm and the last two properties \pref{eq: GE 4} and \pref{eq: bounded norm prod} hold with probability at least $1-\frac{1}{T^3}$.
\end{lemma}
\begin{proof}

To prove \pref{eq: GE 3}, notice that each one of $\hatcovk_{h}$, $m=1,\ldots, M$, is a sum of $N+1$ terms. Furthermore, the $n$-th term of them ($cZ_{n,h}$ in \pref{alg: GR}) has an operator norm upper bounded by $c(1-c\gamma)^{n}$. Therefore, 
\begin{align}
    \norm{\hatcovk_{h}}_{\text{op}} \leq \sum_{n=0}^{N} c(1-c\gamma)^{n} \leq \min\left\{\frac{1}{\gamma}, c(N+1)\right\}\leq \min\left\{\frac{1}{\gamma}, \frac{2}{\lambda}\ln\frac{1}{\epsilon\lambda}\right\}  \label{eq: GE 3 intermediate}
\end{align}
by the definition of $N$ and that $c=\frac{1}{2}$. 
Since $\hatcov_{h}$ is an average of $\hatcovk_{h}$, this implies \pref{eq: GE 3}. 

To show \pref{eq: GE 5}, observe that $\E_t[Y_{n,h}] = \gamma I + \cov_{h}$ and $\{Y_{n,h}\}_{n=1}^N$ are independent. Therefore, we a have
\begin{align*}
    \E\left[\hatcov_{t,h}\right] = \E\left[\hatcovk_{t,h}\right]
    &= cI + c\sum_{i=1}^N \left(I-c\left(\gamma I + \cov_{t,h}\right)\right)^i \\
    &= \left(\gamma I + \cov_{t,h}\right)^{-1} \left( I - \left(I - c\left(\gamma I + \cov_{t,h}\right)\right)^{N+1} \right) 
\end{align*}
where the last step uses the formula: $ \left(I+\sum_{i=1}^{N}A^{i}\right) =(I-A)^{-1} (I-A^{N+1})$ with $A=I-c(\gamma I + \cov_{t,h})$.
Thus, 
\begin{align*}
    \norm{\E_t\left[\hatcov_{h}\right] - \left(\gamma I + \cov_{h}\right)^{-1}}_{\text{op}} 
    &= \norm{\left(\gamma I + \cov_{h}\right)^{-1}\left(I - c\left(\gamma I +  \cov_{h}\right)\right)^{N+1}}_{\text{op}} \\
    &\leq \frac{(1-c(\gamma+\lambda))^{N+1}}{\gamma+\lambda}  \leq \frac{e^{-(N+1)c(\gamma+\lambda)}}{\gamma+\lambda} \leq \epsilon,
\end{align*}
where the first inequality is by $0\prec I-c(\gamma I + I) \preceq I-c(\gamma I + \cov_{h}) \preceq I-c(\gamma + \lambda) I$, and the last inequality is by our choice of $N$ and that $c=\frac{1}{2}$. 

To show \pref{eq: GE 4}, we only further need 
\begin{align*}
    \norm{\hatcov_{h} - \E\left[\hatcov_{h}\right]}_{\text{op}}\leq \epsilon
\end{align*}
and combine it with \pref{eq: GE 5}. 
This can be shown by applying \pref{lem: matrix azuma} with $X_k = \widehat{\Sigma}^{+(k)}_{h}- \E\left[\widehat{\Sigma}^{+(k)}_{h}\right], A_k = \min\left\{\frac{1}{\gamma}, \frac{2}{\lambda}\ln\frac{1}{\epsilon\lambda}\right\}  I$ (recall \pref{eq: GE 3 intermediate} and thus $X_k^2 \preceq A_k^2$), $\sigma=\min\left\{\frac{1}{\gamma}, \frac{2}{\lambda}\ln\frac{1}{\epsilon\lambda}\right\} $, $\tau=\epsilon$, and $n=M$. This gives the following statement: the event $\norm{\hatcov_{h} - \E_t\left[\hatcov_{h}\right]}_{\text{op}}> \epsilon$ holds with probability less than
\begin{align*}
    d\exp\left(-M \times \epsilon^2 \times \frac{1}{8}\times \max\left\{ \gamma^2, \frac{\lambda^2}{4\ln^2\frac{1}{\epsilon\lambda}} \right\}\right)\leq \frac{1}{d^2H^3T^3} \leq \frac{1}{HT^3}
\end{align*}
by our choice of $M$. The conclusion follows by a union bound over $h$. 

To prove \pref{eq: bounded norm prod}, observe that with \pref{eq: GE 4}, we have 
\begin{align*}
    \norm{\hatcov_{h} \cov_{h}}_{\text{op}}
    &\leq  \norm{\left(\gamma I + \cov_{h}\right)^{-1} \cov_{h}}_{\text{op}} + \norm{\left(\hatcov_{h} - \left(\gamma I + \cov_{h}\right)^{-1}\right) \cov_{h}}_{\text{op}} 
    \leq 1 + 2\epsilon
\end{align*}
since $\norm{\cov_{h}}_{\text{op}}\leq 1$. 
\end{proof}

\subsection{The Guarantee of \findpcov}

In this section, we analyze \pref{alg: find policy cover}, which returns a policy cover and its estimated covariance matrices. 
The final guarantee of the policy cover is provided in \pref{lem: property of cover},
but we need to establish a couple of useful lemmas before introducing that.
Note that \pref{alg: find policy cover} bears some similarity with~\citep[Algorithm 1]{wang2020reward} (except for the design of the reward function $r_t$), and thus the analysis is also similar to theirs. 

We first define the following definitions, using notations defined in \pref{alg: find policy cover} and \pref{assum: linear MDP assumption}. 

\begin{definition}
    For any $\pi$ and $m$, define $V^{\pi}_m$ to be the state value function for $\pi$ with respect to reward function $r_m$.
    Precisely, this means $V^{\pi}_m(x_H) = 0$ and for $(x,a)\in X_h\times A$, $h=H-1, \ldots, 0$: $V_m^{\pi}(x) = \sum_{a}\pi(a|x) Q_m^{\pi}(x,a)$ where
    \begin{align*}
    Q^{\pi}_m(x,a) &= r_m(x,a) + \phi(x,a)^\top \theta_{m,h}^{\pi}
\quad\text{and}\quad      \theta^{\pi}_{m,h} = \int_{x'\in X_{h+1}} V^{\pi}_m(x')\nu^{x'}_h \de x'.
    \end{align*}
    Furthermore, let $\pi_m^*$ be the optimal policy satisfying $\pi_m^*= \argmax_{\pi}V_m^{\pi}(x)$ for all $x$,  and define shorthands $V_m^*(x) = V_m^{\pi_m^*}(x)$, $Q_m^*(x,a) = Q_m^{\pi_m^*}(x,a)$, and $\theta^*_{m,h} = \theta^{\pi_m^*}_{m,h}$. 
\end{definition}

The following lemma characterizes the optimistic nature of \pref{alg: find policy cover}.

\begin{lemma}\label{lem:optimistic_policy_cover}
    With probability at least $1-\delta$, for all $h$, all $(x,a)\in X_{h}\times A$, and all $\pi$, \pref{alg: find policy cover} ensures
    \begin{align*}
        0\leq \Qht_m(x,a) - Q_m^\pi(x,a) \leq \E_{x'\sim P(\cdot|x,a)} \left[ \Vht_m(x') - V_m^\pi(x') \right] + 2\xi \|\phi(x,a)\|_{\Gamma_{m,h}^{-1}}. 
    \end{align*}
\end{lemma}
\begin{proof}
    The proof mostly follows that of \citep[Lemma 4]{wei2021learning}.  
    For notational convenience, denote $\phi(x_{\tau,h}, a_{\tau, h})$ as $\phi_{\tau, h}$, and $x'\sim P(\cdot|x_{\tau,h}, a_{\tau,h})$ as $x'\sim (\tau,h)$.  We then have
    \begin{align*}
        &\hattheta_{m,h} - \theta^{\pi}_{m,h} \\
        &= \Gamma_{m,h}^{-1}\left(\frac{1}{N_0}\sum_{\tau=1}^{(m-1)N_0}\phi_{\tau,h}  \Vht_m(x_{\tau, h+1})\right) - \Gamma_{m,h}^{-1}\left(\theta^{\pi}_{m,h} + \frac{1}{N_0}\sum_{\tau=1}^{(m-1)N_0}\phi_{\tau,h}\phi_{\tau,h}^\top \theta^{\pi}_{m,h}\right) \\
        &= \Gamma_{m,h}^{-1}\left(\frac{1}{N_0}\sum_{\tau=1}^{(m-1)N_0}\phi_{\tau,h}  \Vht_m(x_{\tau, h+1})\right) - \Gamma_{m,h}^{-1}\left(\frac{1}{N_0}\sum_{\tau=1}^{(m-1)N_0}\phi_{\tau,h}\E_{x'\sim (\tau,h)}\left[V_m^{\pi}(x')\right]\right) - \Gamma_{m,h}^{-1}\theta^{\pi}_{m,h}  \\
        &= \Gamma_{m,h}^{-1} \left(\frac{1}{N_0}\sum_{\tau=1}^{(m-1)N_0} \phi_{\tau, h} \E_{x'\sim (\tau,h)}\left[\Vht_m(x') - V_m^{\pi}(x')\right] \right) + \zeta_{m,h} - \Gamma_{t,h}^{-1}\theta^{\pi}_{t,h} \tag{define $\zeta_{m,h} = \frac{1}{N_0}\Gamma_{m,h}^{-1}\sum_{\tau=1}^{(m-1)N_0}\left( \Vht_m(x_{\tau,h+1}) - \E_{x'\sim (\tau,h)} \Vht_m(x') \right)$}\\
        &= \Gamma_{m,h}^{-1} \left(\frac{1}{N_0}\sum_{\tau=1}^{(m-1)N_0} \phi_{\tau, h} \phi_{\tau, h}^\top \int_{x'\in X_{h+1}}\nu_h^{x'}\left(\Vht_m(x') - V_m^{\pi}(x')\right)\de x' \right) + \zeta_{m,h} - \Gamma_{m,h}^{-1}\theta^{\pi}_{m,h} \\
        &=  \int_{x'\in X_{h+1}}\nu_h^{x'}\left(\Vht_m(x') - V_m^{\pi}(x')\right) \de x'+ \zeta_{m,h} - \Gamma_{m,h}^{-1}\theta^{\pi}_{m,h} -  \Gamma_{m,h}^{-1}\int_{x'\in X_{h+1}}\nu_h^{x'}\left(\Vht_m(x') -  V_m^{\pi}(x')\right) \de x'.
    \end{align*}
    Therefore, for $x\in X_h$, 
    \begin{align}
        &\Qht_m(x,a) - Q_m^{\pi}(x,a) \nonumber  \\
        &= \phi(x,a)^\top \left(\hattheta_{m,h} - \theta_{m,h}^{\pi}\right) + \xi\|\phi(x,a)\|_{\Gamma_{m,h}^{-1}} \nonumber \\
        &= \phi(x,a)^\top \int_{x'\in X_{h+1}}\nu_h^{x'}\left(\Vht_m(x') - V_m^{\pi}(x')\right) \de x'+ \underbrace{\phi(x,a)^\top\zeta_{m,h}}_{\term_1} + \xi\|\phi(x,a)\|_{\Gamma_{m,h}^{-1}}   \nonumber   \\
        &\qquad \qquad \underbrace{- \phi(x,a)^\top \Gamma_{m,h}^{-1}\int_{x'\in X_{h+1}}\nu_h^{x'}\left(\Vht_m(x') - V_m^{\pi}(x')\right) \de x'}_{\term_2} \underbrace{- \phi(x,a)^\top \Gamma_{m,h}^{-1}\theta^{\pi}_{m,h}}_{\term_3} \nonumber \\
        &= \E_{x'\sim p(\cdot|x,a)} \left[ \Vht_m(x') - V_m^{\pi}(x') \right] + \xi\norm{\phi(x,a)}_{\Gamma_{m,h}^{-1}} + \term_1 + \term_2 + \term_3.
        \label{eq: Q recursive bound}
    \end{align} 
    It remains to bound $|\term_1+\term_2+\term_3|$.
    To do so,  we follow the exact same arguments as in \citep[Lemma 4]{wei2021learning} to bound each of the three terms.  
    
    \paragraph{Bounding $\term_1$. }
    First we have $|\term_1|\leq \norm{\zeta_{m,h}}_{\Gamma_{m,h}}\|\phi(x,a)\|_{\Gamma_{m,h}^{-1}}$. To bound $\norm{\zeta_{m,h}}_{\Gamma_{m,h}}$, we use the exact same argument of \citep[Lemma 4]{wei2021learning} to arrive at (with probability at least $1-\delta$)
    \begin{align}
        \norm{\zeta_{m,h}}_{\Gamma_{m,h}}
        &= \norm{\frac{1}{N_0}\sum_{\tau=1}^{(m-1)N_0} \left(\Vht_m(x_{\tau, h+1}) - \E_{x'\sim (\tau, h)} \Vht_m(x')\right) }_{\Gamma_{m,h}^{-1}}   \nonumber \\
        &\leq 2H\sqrt{\frac{d}{2} \log(M_0+1) + \log\frac{\calN_\varepsilon}{\delta}} + \sqrt{8M_0^2\varepsilon^2},   \label{eq: term 1 bound 1}
    \end{align}
    where $\calN_\varepsilon$ is the $\varepsilon$-cover of the function class that $\Vht_{m}(\cdot)$ lies in. Notice that for all $m$, $\Vht_{m}(\cdot)$ can be expressed as the following: 
    \begin{align*}
        \Vht_{m}(x) = \min\left\{ \max_a \left\{\ramp_{\frac{1}{T}}\left( \|\phi(x,a)\|_{Z}^2 - \frac{\alpha}{M_0}\right) + \xi\|\phi(x,a)\|_Z + \phi(x,a)^\top \theta\right\}, \ \ H \right\}
    \end{align*}
    for some positive definite matrix $Z\in \mathbb{R}^{d\times d}$ with $\frac{1}{1+M_0}I \preceq Z\preceq I$ and vector $\theta\in \mathbb{R}^d$ with $\|\theta\|\leq \sup_{m,\tau,h}\norm{\Gamma_{m,h}^{-1}}_{\text{op}} \times M_0  \|\phi_{\tau,h}\|H \leq M_0 H$. 
    Therefore, we can write the class of functions that $\Vht_m(\cdot)$ lies in as the following set: 
    \begin{align*}
        \calV &= \Bigg\{\min\left\{ \max_a \left\{\ramp_{\frac{1}{T}}\left( \|\phi(x,a)\|_{Z}^2 - \frac{\alpha}{M_0}\right) + \xi\|\phi(x,a)\|_Z + \phi(x,a)^\top \theta\right\}, \ \ H \right\}: \\
            &\qquad \qquad \qquad \qquad \theta\in \mathbb{R}^d: \|\theta\|\leq M_0H, \ \ Z\in \mathbb{R}^{d\times d}: \frac{1}{1+M_0}I \preceq Z\preceq I
        \Bigg\}. 
    \end{align*}
    Now we apply Lemma 12 of \citep{wei2021learning} to $\calV$, with the following choices of parameters: $P=d^2+d$, $\varepsilon= \frac{1}{T^3}$, $B=M_0 H$, and $L=T + \xi \sqrt{1+M_0} + 1\leq 3T$  (without loss of generality, we assume that $T$ is large enough so that the last inequality holds). The value of the Lipschitzness parameter $L$ is according to the following calculation that is similar to~\citep{wei2021learning}: 
    for any $\Delta Z=\epsilon \e_i\e_j^\top$, 
\begin{align*}
    &\frac{1}{|\epsilon|}\left| \sqrt{\phi(x,a)^\top (Z+\Delta Z)\phi(x,a)} -  \sqrt{\phi(x,a)^\top Z\phi(x,a)}\right| \\
    &\leq  \frac{\left|\phi(x,a)^\top \e_i\e_j^\top \phi(x,a)\right|}{\sqrt{\phi(x,a)^\top Z \phi(x,a)}}   \tag{$\sqrt{u+v}-\sqrt{u}\leq \frac{|v|}{\sqrt{u}}$}\\
    &\leq  \frac{\phi(x,a)^\top\left(\frac{1}{2}\e_i\e_i^\top +\frac{1}{2}\e_j\e_j^\top\right)\phi(x,a)}{\sqrt{\phi(x,a)^\top Z \phi(x,a)}} \\
    &\leq  \frac{\phi(x,a)^\top\phi(x,a)}{\sqrt{\phi(x,a)^\top Z \phi(x,a)}}\leq  \sqrt{\frac{1}{\lambda_{\min}(Z)}}\leq \sqrt{1+M_0};
\end{align*}
$\frac{1}{|\epsilon|}\left|\|\phi(x,a)\|_{Z+\Delta Z}^2 - \|\phi(x,a)\|_{Z}^2\right| = |\e_i^\top \phi(x,a) \phi(x,a)^\top \e_j|\leq 1$;
and that $\ramp_{\frac{1}{T}}(\cdot)$ has a slope of $T$ (this is why we need to use the ramp function to approximate an indication function that is not Lipschitz).
    Overall, this leads to 
$\log \calN_{\varepsilon} \leq 20(d^2+d) \log T$.
    Using this fact in \pref{eq: term 1 bound 1}, we get 
    \begin{align*}
        \norm{\zeta_{m,h}}_{\Gamma_{m,h}}\leq 20H\sqrt{d^2\log\left(\frac{T}{\delta}\right)}\leq \frac{1}{3}\xi,  
    \end{align*}
    and thus $|\term_1|\leq \frac{\xi}{3}\|\phi(x,a)\|_{\Gamma_{m,h}^{-1}}$. 
    
    \paragraph{Bounding $\term_2$ and $\term_3$.}
    This is exactly the same as \citep[Lemma 4]{wei2021learning}, and we omit the details. In summary, we can also prove $|\term_2|  \leq \frac{\xi}{3} \norm{\phi(x,a)}_{\Gamma_{m,h}^{-1}}$ and $\term_3 \leq \frac{\xi}{3}\norm{\phi(x,a)}_{\Gamma_{m,h}^{-1}}$.
    
%

    In sum, we can bound
    \begin{align*}
        |\term_1 + \term_2  + \term_3| \leq |\term_1| + |\term_2| + |\term_3| \leq \xi\|\phi(x,a)\|_{\Gamma_{m,h}^{-1}}
    \end{align*}
    for all $m,h$ and $(s,a)$ with probability at least $1-\delta$. 
    
    Combining this with \pref{eq: Q recursive bound}, we get 
    \begin{align}
         \Qht_m(x,a) - Q_m^{\pi}(x,a) 
         &\leq \E_{x'\sim p(\cdot|x,a)} \left[ \Vht_m(x') - V_m^{\pi}(x') \right] + 2\xi\|\phi(x,a)\|_{\Gamma_{m,h}^{-1}},   \label{eq: induction 1} \\
         \Qht_m(x,a) - Q_m^{\pi}(x,a)  
         &\geq \E_{x'\sim p(\cdot|x,a)} \left[ \Vht_m(x') - V_m^{\pi}(x') \right],\label{eq: induction 2}
    \end{align} 
    where \pref{eq: induction 1} proves the second inequality in the lemma. To prove the first inequality in the lemma, we use and induction to show that $\Vht_m(x)\geq V_m^{\pi}(x)$ for all $x$, which combined with \pref{eq: induction 2} finishes the proof. Recall that we define $\Vht_m(x_H)= V_m^{\pi}(x_H)=0$. Assume that $\Vht_m(x)\geq V_m^{\pi}(x)$ holds for $x\in X_{h+1}$. Then by \pref{eq: induction 2}, $\Qht_m(x,a) - Q_m^{\pi}(x,a) \geq 0$ for all $(x,a)\in X_h\times A$. Thus, $\Vht_m(x)- V_m^{\pi}(x) = \max_a\Qht_m(x,a) - \sum_a \pi(a|x) Q_m^\pi(x,a)\geq 0$, finishing the induction. 
\end{proof}

The next lemma provides a ``regret guarantee'' for \pref{alg: find policy cover} with respect to the fake rewards.

\begin{lemma}\label{lem:policy_cover_reg}
    With probability at least $1-2\delta$, \pref{alg: find policy cover} ensures
    \begin{align*}
        \sum_{m=1}^{M_0}V_m^{*}(x_0)- \sum_{m=1}^{M_0}V_m^{\pi_m}(x_0) = \otil\left(d^{3/2}H^2\sqrt{M_0}\right). 
    \end{align*}
\end{lemma}
\begin{proof}
    For any $t\in [(m-1)N_0+1, mN_0]$ and any $h$, 
    \begin{align*}
        &\Vht_m(x_{t,h}) - V_m^{\pi_m}(x_{t,h}) \\
        &= \max_a \Qht_m (x_{t,h}, a) - Q_m^{\pi_m}(x_{t,h}, a_{t,h})  \tag{$\pi_m$ is a deterministic policy} \\
        &= \Qht_m (x_{t,h}, a_{t,h}) - Q_m^{\pi_m}(x_{t,h}, a_{t,h}) \\
        &\leq \E_{x'\sim (x_{t,h}, a_{t,h})} \left[\Vht_m(x') - V_m^{\pi_m}(x')\right] + 2\xi\norm{\phi(x_{t,h},a_{t,h})}_{\Gamma_{m,h}^{-1}} \tag{\pref{lem:optimistic_policy_cover}}\\
        &= \Vht_m(x_{t,h+1}) - V_m^{\pi_m}(x_{t,h+1}) + e_{t,h} + 2\xi\norm{\phi(x_{t,h},a_{t,h})}_{\Gamma_{m,h}^{-1}}.  \tag{define $e_{t,h}$ to be the difference} 
    \end{align*}
    Thus, 
    \begin{align*}
        \Vht_m(x_0) - V_m^{\pi_m}(x_0) \leq \sum_{h} \left(2\xi\norm{\phi(x_{t,h},a_{t,h})}_{\Gamma_{m,h}^{-1}} + e_{t,h}\right).  
    \end{align*}
    Summing over $t$, and using the fact $V^*_m(x_0)\leq \Vht_m(x_0)$ (from \pref{lem:optimistic_policy_cover}), we get 
    \begin{align*}
        &\frac{1}{M_0} \sum_{m=1}^{M_0} \left(V^*_m(x_0) - V_m^{\pi_m}(x_0)\right) \\ 
        &\leq \frac{1}{M_0N_0} \sum_{t=1}^{M_0N_0} \sum_{h} \left(2\xi\norm{\phi(x_{t,h},a_{t,h})}_{\Gamma_{m,h}^{-1}} + e_{t,h}\right) \\
        &\leq \frac{2\xi}{\sqrt{M_0N_0}} \sum_h \sqrt{\sum_{t=1}^{M_0N_0} \|\phi(x_{t,h}, a_{t,h})\|^2_{\Gamma_{m,h}^{-1}}} + \frac{1}{M_0N_0} \sum_{t=1}^{M_0N_0} \sum_{h} e_{t,h}. \tag{Cauchy-Schwarz inequality}
        \end{align*}
        Further using the fact $\sum_{t=1}^{M_0N_0} \|\phi(x_{t,h}, a_{t,h})\|^2_{\Gamma_{m,h}^{-1}} = N_0\sum_{m=1}^{M_0} \inner{\Gamma_{m+1,h}- \Gamma_{m,h}, \Gamma_{m,h}^{-1}} = \otil\left(N_0 d\right)$ (see e.g.,~\citep[Lemma D.2]{jin2020provably}), we bound the first term above by $\otil\left(\xi H\sqrt{d/M_0}\right) = \otil\left(H^2\sqrt{d^3/M_0}\right)$.
        For the second term, note that $\sum_{t=1}^{M_0 N_0} e_{t,h}$ is the sum of a martingale difference sequence.
        By Azuma's inequality, the entire second term is thus of order $\otil\left(\frac{H^2\log(1/\delta)}{\sqrt{M_0N_0}}\right)$ with probability at least $1-\delta$.
        This finishes the proof.
\end{proof}

Finally, we are ready to show the guarantee of the returned policy cover.
Recall our definition of known state set:
\[
\calK = \left\{x \in X: \forall a\in A, \|\phi(x,a)\|_{(\widehat{\cov}_h^{\textcov})^{-1}}^2 \leq \alpha \text{ where $h$ is such that $x \in X_h$} \right\}.
\]

\begin{lemma}\label{lem: property of cover}
    For any $h = 0, \ldots, H-1$, with probability at least $1-4\delta$ (over the randomness in the first $T_0$ rounds), the covariance matrices $\widehat{\cov}_h^{\textcov}$ returned by \pref{alg: find policy cover} satisfies that for any policy $\pi$, 
    \[
    \Pr_{x_h \sim \pi}\left[x_h \notin \calK \right] \leq \otil\left(\frac{dH}{\alpha} \right).
    \]
    where $x_h \in X_h$ is sampled from executing $\pi$.
\end{lemma}
\begin{proof}
    We define an auxiliary policy $\pi'$ which only differs from $\pi$ for unknown states in layer $h$. Specifically, for $x\in X_h$ not in $\calK$, let $a$ be such that $\|\phi(x,a)\|_{(\widehat{\cov}_h^{\textcov})^{-1}}^2 \geq \alpha$ (which must exist by the definition of $\calK$),
    then $\pi'(a'|x) = \ind[a'=a]$ for all $a' \in A$.
    By doing so, we have
    \begin{align*}
        & \Pr_{x_h \sim \pi}\left[x_h \notin \calK  \right] \\
        &= \Pr_{(x_h, a) \sim \pi'}\left[\|\phi(x_h,a)\|_{(\widehat{\cov}_h^{\textcov})^{-1}}^2 \geq \alpha  \right] \\
        &= \Pr_{(x_h,a) \sim \pi'}\left[\|\phi(x_h,a)\|_{\Gamma_{M_0+1, h}^{-1}}^2 \geq \frac{\alpha}{M_0} \right] \\
        &\leq \frac{1}{M_0} \sum_{m=1}^{M_0}\Pr_{(x_h,a) \sim \pi'}\left[\|\phi(x_h,a)\|_{\Gamma_{m, h}^{-1}}^2 \geq \frac{\alpha}{M_0} \right] \tag{$\Gamma_{m, h} \preceq  \Gamma_{M_0+1, h}$} \\
        &\leq \frac{1}{M_0}
        \sum_{m=1}^{M_0}\E_{(x_h,a)\sim \pi'}\left[ \ramp_{\frac{1}{T}}\left(\|\phi(x,a)\|_{\Gamma_{m, h}^{-1}}^2 - \frac{\alpha}{M_0} \right)\right] \tag{$\one[y\geq 0]\leq \ramp_z(y)$} \\
        &\leq \frac{1}{M_0} \sum_{m=1}^{M_0} V_m^{\pi'}(x_0)  \tag{rewards $r_m(\cdot, \cdot)$ are non-negative} \\
        &\leq \frac{1}{M_0} \sum_{m=1}^{M_0}V_m^{\pi_m}(x_0) + \frac{1}{M_0}\times \otil\left(d^{3/2}H^2\sqrt{M_0}\right) \tag{\pref{lem:policy_cover_reg}}\\
        &\leq \frac{1}{M_0N_0}\sum_{t=1}^{M_0N_0} \sum_{h=0}^{H-1} r_{m}(x_{t,h}, a_{t,h}) + \otil\left(\frac{H}{\sqrt{M_0 N_0}}\right) + \otil\left(\frac{d^{3/2}H^2}{\sqrt{M_0}}\right)   \tag{by Azuma's inequality}\\
        &\leq \frac{1}{M_0N_0}\times \frac{1}{\frac{\alpha}{M_0}} \sum_{t=1}^{M_0N_0}\sum_{h=0}^{H-1}\|\phi(x_{t,h},a_{t,h})\|_{\Gamma_{m, h}^{-1}}^2 +  \otil\left( \frac{d^{3/2}H^2}{\sqrt{M_0}} \right) \tag{$\ramp_z(y-y')\leq \frac{y}{y'}$ for $y>0,y'>z>0$}\\
        &\leq \frac{1}{M_0N_0}\times \frac{1}{\frac{\alpha}{M_0}} \times \otil\left(N_0 dH\right) + \otil\left(\frac{d^{3/2}H^2}{\sqrt{M_0}}\right)  \tag{same calculation as done in the proof of \pref{lem:policy_cover_reg}} \\
        &\leq \otil\left(\frac{dH}{\alpha}  + \frac{d^{3/2}H^2}{\sqrt{M_0}}\right). 
    \end{align*}
    Finally, using the definition of $M_0$ finishes the proof. 
\end{proof}

\section{Details Omitted in \pref{sec: linear Q}}\label{app: linear Q appendix}

In this section, we analyze \pref{alg: linear Q} and prove \pref{thm: linear Q theorem}. 
In the analysis, we require that $\pi_t(a|x)$ and $\bonusQ_t(x,a)$ are defined for all $x, a, t$, but in \pref{alg: linear Q},  they are only explicitly defined if the learner has ever visited state $x$. Below, we construct a virtual process that is equivalent to \pref{alg: linear Q}, but with all $\pi_t(a|x)$ and $\bonusQ_t(x,a)$ well-defined. 

Imagine a virtual process where at the end of episode $t$ (the moment when $\hatcov_t$ has been defined), $\Bonus(t,x,a)$ is called once for every $(x,a)$, in an order from layer $H-1$ to layer $0$. Observe that within $\Bonus(t,x,a)$, other $\Bonus(t', x', a')$ might be called, but either $t'<t$, or $x'$ is in a later layer. Therefore, in this virtual process, every recursive call will soon be returned in the third line of \pref{alg: generating B samples} because they have been called previously and the values of them are already determined. Given that $\Bonus(t,x,a)$ are all called once, at the beginning of episode $t+1$, $\pi_{t+1}$ will be well-defined for all states since it only depends on $\Bonus(t',x',a')$ with $t'\leq t$ and other quantities that are well-defined before episode $t+1$. 

Comparing the virtual process and the real process, we see that the virtual process calculates all entries of $\Bonus(t,x,a)$, while the real process only calculates a subset of them that are necessary for constructing $\pi_t$ and $\hatcov_t$. However, they define exactly the same policies as long as the random seeds we use for each entry of $\Bonus(t,x,a)$ are the same for both processes. Therefore, we can define $\bonusQ_t(x,a)$ unambiguously as the value returned by $\Bonus(t,x,a)$ in the virtual process, and $\pi_t(a|x)$ as shown in \eqref{eq: linear Q policy} with $\Bonus(\tau,x,a)$ replaced by $\bonusQ_\tau(x,a)$.


Now, we follow the exactly same regret decomposition as described in \pref{sec: tabular}, with the new definition of $\Qht_t(x,a)\triangleq \phi(x,a)^\top \hattheta_{t,h}$ (for $x \in X_h$) and $\bonusQ_t(x,a)$ described above:
\begin{align}
	&\sum_{t=1}^T\sum_{h=0}^{H-1}\E_{x_h\sim\pistar}\left[\inner{\pi_t(\cdot|x_h)-\pistar(\cdot|x_h), \Q_t^{\pi_t}(x_h,\cdot) -B_t(x_h,\cdot)}\right] \nonumber\\
	&= \underbrace{\sum_{t=1}^T\sum_{h=0}^{H-1}\E_{x_h\sim\pistar}\left[\inner{\pi_t(\cdot|x_h), \Q_t^{\pi_t}(x_h,\cdot)-\Qht_t(x_h,\cdot)}\right]}_{\bias}+\underbrace{\sum_{t=1}^T\sum_{h=0}^{H-1}\E_{x_h\sim\pistar}\left[\inner{\pistar(\cdot|x_h), \Qht_t(x_h,\cdot)-\Q_t^{\pi_t}(x_h,\cdot)}\right]}_{\biastwo}\nonumber\\
	&\quad
	 + \underbrace{\sum_{t=1}^T\sum_{h=0}^{H-1}\E_{x_h\sim\pistar}\left[\inner{\pi_t(\cdot|x_h)-\pistar(\cdot|x_h), \Qht_t(x_h,\cdot)-B_t(x_h,\cdot)}\right]}_{\regterm}. \nonumber
\end{align}

We then bound $\E[\bias + \biastwo]$ and $\E[\regterm]$ in \pref{lem: linear Q bias term} and \pref{lem: linear Q regret term} respectively. \\


\begin{lemma}
\label{lem: linear Q bias term}
     If $\beta \leq H$, then $\E[\bias + \biastwo]$ is upper bounded by
     \begin{align*}
         &\frac{\beta}{4} \E\left[\sum_{t=1}^T \sum_{h=0}^{H-1} \E_{x_h\sim\pistar}\left[ \sum_a  \big(\pi_t(a|x_h)+\pistar(a|x_h)\big) \|\phi(x_h,a)\|_{\hatcov_{t,h}}^2\right]\right]  +  \order\left(\frac{\gamma dH^3 T}{\beta} + \epsilon H^2 T\right). 
     \end{align*}
\end{lemma}

\begin{proof}[Proof of \pref{lem: linear Q bias term}]
    Consider a specific $(t,x,a)$. Let $h$ be such that $x\in X_{h}$. Then we proceed as
    \begin{align}
        &\E_t\left[Q^{\pi_t}_t(x,a) - \Qht_t(x,a)\right]  \nonumber  \\
        &=\phi(x,a)^\top \left(\theta_{t,h}^{\pi_t} - \E_t\left[\hattheta_{t,h}\right]\right) \nonumber \\ 
        &= \phi(x,a)^\top \left(\theta_{t,h}^{\pi_t} - \E_t\left[\hatcov_{t,h}\right]\E_t\left[ \phi(x_{t,h},a_{t,h})L_{t,h} \right]\right) \nonumber \tag{definition of $\hattheta_{t,h}$}\\
        &=  \phi(x,a)^\top \left(\theta_{t,h}^{\pi_t} - \left(\gamma I + \cov_{t,h}\right)^{-1}\E_t\left[  \phi(x_{t,h},a_{t,h})L_{t,h} \right]\right)  + \order(\epsilon H)  \tag{by \pref{eq: GE 5} of \pref{lem: GR lemma contingency} and that $\|\phi(x,a)\|\leq 1$ for all $x,a$ and $L_{t,h}\leq H$} \\
        &= \phi(x,a)^\top \left(\theta_{t,h}^{\pi_t} - \left(\gamma I + \cov_{t,h}\right)^{-1}\cov_{t,h} \theta_{t,h}^{\pi_t}\right) + \order(\epsilon H) \tag{$\E[L_{t,h}]=\phi(x_{t,h}, a_{t,h})^\top \theta^{\pi_t}_{t,h}$} \\
        &= \gamma  \phi(x,a)^\top (\gamma I + \cov_{t,h})^{-1} \theta_{t,h}^{\pi_t}+ \order(\epsilon H)\nonumber \tag{$\theta_{t,h}^{\pi_t} =   \left(\gamma I + \cov_{t,h}\right)^{-1}\left(\gamma I + \cov_{t,h}\right)\theta_{t,h}^{\pi_t}$} \\
        &\leq  \gamma\|\phi(x,a)\|_{(\gamma I + \cov_{t,h})^{-1}}^2    \|\theta_{t,h}^{\pi_t}\|_{(\gamma I + \cov_{t,h})^{-1}}^2 + \order(\epsilon H) \nonumber \tag{Cauchy-Schwarz inequality} \\
        &\leq \frac{\beta}{4}  \|\phi(x,a)\|_{(\gamma I + \cov_{t,h})^{-1}}^2  + \frac{\gamma^2}{\beta}  \|\theta_{t,h}^{\pi_t}\|_{(\gamma I + \cov_{t,h})^{-1}}^2 + \order(\epsilon H) \nonumber \tag{AM-GM inequality}\\
        &\leq \frac{\beta}{4} \E_t\left[  \|\phi(x,a)\|_{\hatcov_{t,h}}^2\right] + \frac{\gamma dH^2}{\beta} + \order\left(\epsilon (H+\beta)\right)    \label{eq: long proff 2}
    \end{align}
    where in the last inequality we use \pref{eq: GE 5} again and also $\|\theta^{\pi}_{t,h}\|^2 \leq dH^2$ according to \pref{assum: linear Q assumption}. 
    Taking expectation over $x$ and summing over $t,a$ with weights $\pi_t(a|x)$, we get  
    \begin{align*}
        \E[\bias] \leq \frac{\beta}{4} \E\left[\sum_{t=1}^T \sum_{h=0}^{H-1}\E_{x_h \sim \pistar}\left[ \sum_a \pi_t(a|x_h)  \|\phi(x_h,a)\|_{\hatcov_{t,h}}^2\right]\right]  +  \order\left(\frac{\gamma dH^3T}{\beta} + \epsilon H^2T\right).  \tag{using $\beta\leq H$} 
    \end{align*}
    By the same argument, we can show that $\E_t[\Qht_t(x,a) - Q_t^{\pi_t}(x,a)]$ is also upper bounded by the right-hand side of \pref{eq: long proff 2}, 
    and thus
    \begin{align*}
        \E[\biastwo] \leq \frac{\beta}{4} \E\left[\sum_{t=1}^T \sum_{h=0}^{H-1} \E_{x_h \sim \pistar}\left[ \sum_a \pistar(a|x_h)  \|\phi(x_h,a)\|_{\hatcov_{t,h}}^2\right]\right]  +  \order\left(\frac{\gamma dH^3T}{\beta} + \epsilon H^2T \right).
    \end{align*}

    Summing them up finishes the proof. 
\end{proof}

\begin{lemma}
\label{lem: linear Q regret term}
If $\eta \beta \leq \frac{\gamma}{12H^2}$ and  $\eta\leq \frac{\gamma}{2H}$, then $\E[\regterm]$ is upper bounded by
\begin{align*} 
    &\frac{H\ln |A|}{\eta} + 2\eta H^2\E\left[\sum_{t=1}^T \sum_{h=0}^{H-1}\E_{x_h \sim \pistar}\left[ \sum_a \pi_t(a|x_h)  \|\phi(x_h,a)\|_{\hatcov_{t,h}}^2\right]\right] 
    \\
    &\qquad \qquad + \frac{1}{H}\E\left[\sum_{t=1}^T \sum_{h=0}^{H-1} \E_{x_h \sim \pistar}\left[ \sum_a \pi_t(a|x_h) \bonusQ_t(x,a)\right] \right] + \order\left(\eta\epsilon H^3 T + \frac{\eta H^3}{\gamma^2T^2}\right).  
\end{align*}
\end{lemma}

\begin{proof}[Proof of \pref{lem: linear Q regret term}]
    Again, we will apply the regret bound of the exponential weight algorithm \pref{lem: exponential weight lemma} to each state.  We start by checking the required condition: $\eta |\phi(x,a)^\top \hattheta_{\tau,h} - \bonusQ_t(x,a)|\leq 1$. This can be seen by that
    \begin{align*}
        \eta \left|\phi(x,a)^\top \hattheta_{\tau, h}\right| 
        &= \eta \left| \phi(x,a)^\top \hatcov_{t,h} \phi(x_{t,h},a_{t,h})L_{t,h} \right| \\
        &\leq \eta \times \norm{\hatcov_{t,h}}_{\text{op}}\times L_{t,h} 
        \leq \frac{\eta H}{\gamma}
        \leq \frac{1}{2},  \tag{\pref{eq: GE 3} and the condition $\eta\leq \frac{\gamma}{2H}$}
    \end{align*}
    and that by the definition of $\Bonus(t,x,a)$, we have 
    \begin{align}
        \eta\bonusQ_t(x,a) \leq \eta\times H\left(1+\frac{1}{H}\right)^H \times  2\beta\sup_{x,a,h} \|\phi(x,a)\|_{\hatcov_{t,h}}^2 
        \leq \frac{6\eta\beta H}{\gamma}
        \leq \frac{1}{2H},\label{eq: eta B linear Q} 
    \end{align}
    where the last inequality is by \pref{eq: GE 3} again and the condition $\eta \beta \leq \frac{\gamma}{12H^2}$.

    Thus, by \pref{lem: exponential weight lemma}, we have for any $x$, 
    \begin{align}
        &\E\left[\sum_{t=1}^T \sum_a \left(\pi_t(a|x) - \pistar(a|x)\right)\Qht_t(x,a)\right] \nonumber \\ 
        &\leq \frac{\ln |A|}{\eta } + 2\eta \E\left[\sum_{t=1}^T \sum_a  \pi_t(a|x)\Qht_t(x,a)^2\right] + 2\eta \E\left[\sum_{t=1}^T \sum_a  \pi_t(a|x)\bonusQ_t(x,a)^2\right].  \label{eq: regret Q tmp}
    \end{align}
    The last term in \pref{eq: regret Q tmp} can be upper bounded by $ \E\left[\frac{1}{H}\sum_{t=1}^T \sum_a  \pi_t(a|x)\bonusQ_t(x,a)\right]$ because  $\eta \bonusQ_t(x,a)\leq  \frac{1}{2H}$ as we verified in \pref{eq: eta B linear Q}. To bound the second term in \pref{eq: regret Q tmp}, we use the following: for $(x,a)\in X_h\times A$, 
    \begin{align*}
        \E_t\left[\Qht_t(x,a)^2\right]
        &\leq H^2\E_t\left[\phi(x,a)^\top \hatcov_{t,h}\phi(x_{t,h}, a_{t,h})\phi(x_{t,h}, a_{t,h})^\top \hatcov_{t,h} \phi(x,a)\right]\\
        &= H^2\E_t\left[\phi(x,a)^\top \hatcov_{t,h}\cov_{t,h} \hatcov_{t,h} \phi(x,a)\right] \\ 
        &\leq H^2\E_t\left[\phi(x,a)^\top \hatcov_{t,h}\cov_{t,h} \left(\gamma I + \cov_{t,h}\right)^{-1}\phi(x,a)\right] + \order\left(\epsilon H^2 + \frac{H^2}{\gamma^2 T^3}\right)   \tag{$*$}\\
        &\leq H^2\phi(x,a)^\top \left(\gamma I + \cov_{t,h}\right)^{-1}\cov_{t,h} \left(\gamma I + \cov_{t,h}\right)^{-1}\phi(x,a) + \order\left(\epsilon H^2 + \frac{H^2}{\gamma^2 T^3}\right)  \tag{by \pref{eq: GE 5}} \\
        &\leq H^2\phi(x,a)^\top \left(\gamma I + \cov_{t,h}\right)^{-1} \phi(x,a) + \order\left(\epsilon H^2+ \frac{H^2}{\gamma^2 T^3}\right)\\ 
        &\leq H^2\E_t\left[\phi(x,a)^\top \hatcov_{t,h} \phi(x,a)\right] + \order\left(\epsilon H^2 + \frac{H^2}{\gamma^2 T^3}\right) \tag{by \pref{eq: GE 5} again} \\
        &= H^2\E_t\left[\|\phi(x,a)\|_{\hatcov_{t,h}}^2 \right] + \order\left(\epsilon H^2 + \frac{H^2}{\gamma^2 T^3}\right)
    \end{align*}
    where $(*)$ is because by \pref{eq: GE 4} and \pref{eq: bounded norm prod}, $\|(\gamma I+\cov_{t,h})^{-1}-\hatcov_{t,h}\|_{\text{op}}\leq 2\epsilon$ and $\|\hatcov_{t,h}\cov_{t,h}\|_{\text{op}}\leq 1+2\epsilon$ hold with probability $1-\frac{1}{T^3}$; for the remaining probability, we upper bound $H^2\phi(x,a)^\top \hatcov_{t,h}\cov_{t,h}\hatcov_{t,h}\phi(x,a)$ by $\frac{H^2}{\gamma^2}$. 
    Combining them with \pref{eq: regret Q tmp} and taking expectation over states finishes the proof. 
\end{proof}

With \pref{lem: linear Q bias term} and \pref{lem: linear Q regret term}, we can now prove \pref{thm: linear Q theorem}. 

\begin{proof}[Proof of \pref{thm: linear Q theorem}]

Combining \pref{lem: linear Q bias term} and \pref{lem: linear Q regret term}, we get (under the required conditions of the parameters):
\begin{align*}
    &\E\left[\bias + \biastwo + \regterm\right] \\
    &\leq \order\left(\frac{H\ln |A|}{\eta} + \frac{\gamma d H^3 T}{\beta} + \epsilon H^2 T + \eta\epsilon H^3 T + \frac{\eta H^3}{\gamma^2T^2}\right) \\
    &\qquad \qquad + \left(2\eta H^2 + \frac{\beta}{4}\right)\E\left[\sum_{t=1}^T \sum_{h=0}^{H-1}\E_{x_h\sim\pistar}\left[\sum_a \Big(\pi_t(a|x_h) + \pistar(a|x_h)\Big)\|\phi(x_h,a)\|_{\hatcov_{t,h}}^2\right]\right] \\ 
    &\qquad \qquad + \frac{1}{H}\E\left[\sum_{t=1}^T\sum_{h=0}^{H-1}\E_{x_h\sim \pistar}\left[ \sum_{a}\pi_t(a|x_h)\bonusQ_t(x_h,a)\right]\right].  
\end{align*}

We see that \pref{eq:modified bound expected} is satisfied in expectation as long as we have $2\eta H^2 + \frac{\beta}{4}\leq \beta$ and define $\bonus_t(x,a)\triangleq \beta \|\phi(x,a)\|_{\hatcov_{t,h}}^2 + \beta \sum_{a'}\pi_t(a'|x)\|\phi(x,a')\|_{\hatcov_{t,h}}^2$ (for $x\in X_h$).  By the definition of \pref{alg: generating B samples}, \pref{eq: expected condition 2} is also satisfied with this choice of $\bonus_t(x,a)$. Therefore, 
we can apply \pref{lem: expected version} to obtain a regret bound.
To simply the presentation, we first pick $\epsilon = \frac{1}{H^3 T}$ so that all $\epsilon$-related terms become $\order(1)$.
Then we have
\begin{align*}
    &\E[\Reg]  \\
    &=\otil\left(\frac{H}{\eta} + \frac{\gamma d H^3 T}{\beta} +  \frac{\eta H^3}{\gamma^2T^2} + \E\left[\sum_{t=1}^T\sum_{h=0}^{H-1}\E_{(x_h,a)\sim\pi_t} \left[\bonus_t(x,a)\right]\right]\right) \\
    &= \otil\left(\frac{H}{\eta} + \frac{\gamma d H^3 T}{\beta} +  \frac{\eta H^3}{\gamma^2T^2} + \beta\E\left[\sum_{t=1}^T \sum_{h=0}^{H-1}\E_{(x_h,a)\sim\pi_t} \left[\|\phi(x,a)\|_{\hatcov_{t,h}}^2\right]\right]\right) \\
    &= \otil\left(\frac{H}{\eta} + \frac{\gamma d H^3 T}{\beta} +  \frac{\eta H^3}{\gamma^2T^2} + \beta\E\left[\sum_{t=1}^T \sum_{h=0}^{H-1}\E_{(x_h,a)\sim\pi_t} \left[\|\phi(x,a)\|_{(\gamma I + \cov_{t,h})^{-1}}^2\right]\right]\right)
    \tag{\pref{eq: GE 5} and $\beta \leq H$} \\
    & = \otil\left(\frac{H}{\eta} +  \frac{\gamma d H^3 T}{\beta} + \frac{\eta H^3}{\gamma^2T^2} + \beta dHT \right),
\end{align*}
where the last step uses the fact 
\begin{align}
\E_t\left[\sum_h \E_{(x_h,a)\sim\pi_t} \left[\|\phi(x,a)\|_{(\gamma I + \cov_{t,h})^{-1}}^2\right]\right]
&\leq \E_t\left[\sum_h \E_{(x_h,a)\sim\pi_t} \left[\|\phi(x,a)\|_{\cov_{t,h}^{-1}}^2\right]\right] \nonumber \\
&= \sum_h \inner{\cov_{t,h}, \cov_{t,h}^{-1}} = dH. \label{eq:LB_stability}
\end{align}

Finally, choosing the parameters under the specified constraints as: 
\begin{align*}
    \gamma &= (dT)^{-\frac{2}{3}}, \qquad
    \beta = H(dT)^{-\frac{1}{3}}, \qquad 
    \epsilon = \frac{1}{H^3 T}, \\
    \eta &= \min\left\{ \frac{\gamma}{2H}, \frac{3\beta}{8H^2}, \frac{\gamma}{12\beta H^2} \right\},
\end{align*}
we further bound the regret by  $\otil\left(H^2(dT)^{\frac{2}{3}} + H^4(dT)^{\frac{1}{3}}\right)$.  
\end{proof}

\section{Details Omitted in \pref{sec: linear MDP}}\label{app: linear MDP appendix}

In this section, we analyze our algorithm for linear MDPs.
First, we show the main benefit of exploring with the policy cover, 
that is, it ensures a small magnitude for $\bonus_t(x,a)$, as shown below.

\begin{lemma} \label{lem: bonus bounded by 1}
    If  $\gamma\geq \frac{36\beta^2}{\explore}$ and $\beta\epsilon \leq \frac{1}{8}$, then $\bonus_k(x,a) \leq 1$ for all $(x,a)$ and all $k$ (with high probability). 
\end{lemma}
\begin{proof}
   According to the definition of $\bonus_k(x,a)$ (in \pref{alg:linearMDP}), it suffices to show that for $x\in \calK$, $\beta\|\phi(x,a)\|_{\hatcov_{k,h}}^2 \leq \frac{1}{2}$ for any $a$.  
To do so, note that the \GR procedure ensures that $\hatcov_{k,h}$ is an estimation of the inverse of $\gamma I + \cov_{k,h}^{\mix}$, where 
\begin{equation}\label{eq:mix_covariance}
\cov_{k,h}^{\mix} = \explore \cov_h^{\textcov} + (1-\explore)\E_{(x_h,a)\sim \pi_k}[\phi(x_h,a)\phi(x_h,a)^\top]
\end{equation}
and $\cov_h^{\textcov} = \frac{1}{M_0}\sum_{m=1}^{M_0}  \E_{(x_h,a) \sim \pi_m} \left[\phi(x_h,a)\phi(x_h,a)^\top\right]$ is the covariance matrix of the policy cover $\pcov$.
By \pref{eq: GE 4}, we have with probability at least $1-1/T^3$,
\[
\beta \|\phi(x,a)\|_{\hatcov_{k,h}}^2 \leq \beta \|\phi(x,a)\|_{(\gamma I + \cov_{k,h}^{\mix})^{-1}}^2 + 2\beta \epsilon
\leq \beta \|\phi(x,a)\|_{(\gamma I + \cov_{k,h}^{\mix})^{-1}}^2 + \frac{1}{4}.
\]
The first term can be further bounded as
$
\frac{\beta}{\explore}\|\phi(x,a)\|_{(\frac{\gamma}{\explore}I + \cov^{\textcov}_h)^{-1}}^2  
\leq \frac{\beta}{\explore}\|\phi(x,a)\|_{(\frac{1}{M_0}I + \cov^{\textcov}_h)^{-1}}^2   
$,
where the last step is because $\frac{\gamma}{\explore}M_0 \geq  \frac{\gamma}{\explore} \times \frac{\explore^2}{36\beta^2} \geq 1$ by our condition.   
Finally, we show that $\frac{1}{M_0}I + \cov^{\textcov}_h$ and $\widehat{\cov}_h^{\textcov}$ are close.
Recall the definition of the latter:
    \begin{align*}
        \widehat{\cov}_h^{\textcov} = \frac{1}{M_0}I + \frac{1}{M_0N_0}\sum_{m=1}^{M_0} \sum_{t=(m-1)N_0+1}^{mN_0} \phi(x_{t,h},a_{t,h})\phi(x_{t,h},a_{t,h})^\top.  
    \end{align*}
    We now apply \pref{lem: matrix azuma} with $n=N_0$ and
    \begin{align*}
        X_k = \frac{1}{M_0}\sum_{m=1}^{M_0}\phi(x_{\tau(m,k), h}, a_{\tau(m,k), h})\phi(x_{\tau(m,k), h}, a_{\tau(m,k), h})^\top - \frac{1}{M_0}\sum_{m=1}^{M_0}\E_{(x_h,a) \sim \pi_m} \left[\phi(x_h,a)\phi(x_h,a)^\top\right], 
    \end{align*}
    for $k=1,\ldots, N_0$, where $\tau(m,k)\triangleq (m-1)N_0+k$. Note that $X_k^2 \preceq I$. Therefore, we can pick $A_k=I$ and $\sigma=1$. By \pref{lem: matrix azuma}, we have with probability at least $1-\delta$, 
    \begin{align*}
        \norm{\widehat{\cov}^{\textcov}_h - \frac{1}{M_0}I - \cov_h^{\textcov} }_{\text{op}} \leq \sqrt{\frac{8\log(d/\delta)}{N_0}}. 
    \end{align*} 
    Following the same proof as \citep[Theorem 2.1]{meng2010optimal}, we have 
    \begin{align*}
        \norm{\left(\frac{1}{M_0}I + \cov_h^{\textcov}\right)^{-1} - \left(\widehat{\cov}_h^{\textcov}\right)^{-1}}_{\text{op}} \leq M_0^2 \norm{\widehat{\cov}^{\textcov}_h - \frac{1}{M_0}I - \cov_h^{\textcov} }_{\text{op}} \leq M_0^2 \sqrt{\frac{8\log(d/\delta)}{N_0}} \leq \frac{\alpha}{2}. \tag{by our choice of $N_0$ and $M_0$}  
    \end{align*}
    Consequently, for any vector $\phi$ with $\|\phi\|\leq 1$, we have 
    \begin{equation*}
        \left|\norm{\phi}_{ \left(\frac{1}{M_0}I + \cov_h^{\textcov}\right)^{-1}}^2 - \norm{\phi}_{ \left( \widehat{\cov}_h^{\textcov}\right)^{-1}}^2 \right|
        \leq  \norm{\left(\frac{1}{M_0}I + \cov_h^{\textcov}\right)^{-1} - \left(  \widehat{\cov}_h^{\textcov}\right)^{-1}}_{\text{op}}  \leq \frac{\alpha}{2}. 
    \end{equation*}
    Therefore, combining everything we have
    \[
    \beta \|\phi(x,a)\|_{\hatcov_{k,h}}^2 
    \leq \frac{\beta}{\explore}\left(\|\phi(x,a)\|_{\left(  \widehat{\cov}_h^{\textcov}\right)^{-1}}^2+\frac{\alpha}{2}\right) + \frac{1}{4}
    \leq \frac{3\beta\alpha}{2\explore}+ \frac{1}{4} \leq \frac{1}{2},
    \]
where the last two steps use the fact $x \in \calK$ and the value of $\alpha$.
This finishes the proof.
\end{proof}


Next, we define the following notations for convenience due to the epoch schedule of our algorithm, and then proceed to prove the main theorem.

\begin{definition}
\begin{align*}
    \barell_{k}(x,a) &=  \frac{1}{\len}\sum_{t=T_0+(k-1)\len+1}^{T_0+k\len}\ell_t(x,a)\\
    \barQ_{k}^{\pi}(x,a) &=  Q^{\pi}(x,a;\barell_k)\\
    \bartheta_{k,h}^{\pi} 
    &~~\text{is such that}~~ \barQ_{k}^{\pi}(x,a) = \phi(x,a)^\top \bartheta_{k,h}^{\pi} \\ 
        \bonusQ_k(x,a) &= \bonus_k(x,a) + \left(1+\frac{1}{H}\right)\E_{x'\sim P(\cdot|x,a)}\E_{a'\sim \pi_k(\cdot|x')}[\bonusQ_k(x',a')] \\
        \hatbonusQ_k(x,a) &= \bonus_k(x,a) + \phi(x,a)^\top \bonusVec_{k,h}   \qquad \qquad \text{{\normalfont (for $x\in X_h$)}}
    \end{align*}
\end{definition}
 
\begin{proof}[Proof of \pref{thm: linear MDP theorem}] 
 
We first analyze the regret of policy optimization after the first $T_0$ rounds.
Our goal is again to prove \pref{eq:modified bound expected} which in this case bounds
\[
\sum_{k=1}^{(T-T_0)/\len}\sum_h\E_{X_h \ni x\sim\pistar}\left[ \sum_a\Big(\pi_k(a|x)-\pistar(a|x)\Big) \left(\barQ^{\pi_k}_k(x,a)-\bonusQ_k(x, a)\right)\right].
\]
The first step is to separate known states and unknown states.
For unknown states, we have
\begin{align*}
&\sum_{k=1}^{(T-T_0)/\len}\sum_h\E_{X_h \ni x\sim\pistar}\left[\one[x\notin \calK] \sum_a\Big(\pi_k(a|x)-\pistar(a|x)\Big) \left(\barQ^{\pi_k}_k(x,a)-\bonusQ_k(x, a)\right)\right] \\
&\leq \frac{(T-T_0)He}{\len} \sum_h\E_{X_h \ni x\sim\pistar}\left[\one[x\notin \calK]\right] 
= \otil\left(\frac{dH^3T}{\alpha\len}\right),
\end{align*}
where the first step is by the facts $0 \leq \barQ^{\pi_k}_k(x,a) \leq H$ and $0 \leq \bonusQ_k(x, a) \leq (1+\frac{1}{H})^H \times H \leq He$ (\pref{lem: bonus bounded by 1}),
and the second step applies \pref{lem: property of cover}.
For known states, we apply a similar decomposition as previous analysis,
but since we also use function approximation for bonus $\bonusQ_t(x,a)$, we need to account for its estimation error, which results in two extra bias terms:
\begin{align*}
	&\sum_{k=1}^{(T-T_0)/\len}\sum_h\E_{X_h \ni x\sim\pistar}\left[\one[x\in \calK] \sum_a\Big(\pi_k(a|x)-\pistar(a|x)\Big) \left(\barQ^{\pi_k}_k(x,a)-\bonusQ_k(x, a)\right)\right]  \\
	&= \underbrace{\sum_{k=1}^{(T-T_0)/\len}\sum_h\E_{X_h \ni x\sim\pistar}\left[\one[x\in \calK] \sum_a \pi_k(a|x)\Big( \barQ_k^{\pi_k}(x,a)-\Qht_k(x,a)\Big)\right]}_{\bias}  \\
	&\quad +\underbrace{\sum_{k=1}^{(T-T_0)/\len}\sum_h\E_{X_h \ni x\sim\pistar}\left[\one[x\in \calK] \sum_a \pistar(a|x) \Big(\Qht_k(x,a)-\barQ_k^{\pi_k}(x,a)\Big)\right]}_{\biastwo}\\
	&\quad +\underbrace{\sum_{k=1}^{(T-T_0)/\len}\sum_h\E_{X_h \ni x\sim\pistar}\left[\one[x\in \calK] \sum_a \pi_k(a|x)\Big( \hatbonusQ_k(x,\cdot)-\bonusQ_k(x,\cdot)\Big)\right]}_{\biasthree} \\
	&\quad+ \underbrace{\sum_{k=1}^{(T-T_0)/\len}\sum_h\E_{X_h \ni x\sim\pistar}\left[\one[x\in \calK] \sum_a \pistar(a|x)\Big( \bonusQ_k(x,a)-\hatbonusQ_k(x,a)\Big)\right]}_{\biasfour}\\ 
	&\quad 
	 + \underbrace{\sum_{k=1}^{(T-T_0)/\len}\sum_h\E_{X_h \ni x\sim\pistar}\left[\one[x\in \calK] \sum_a \Big(\pi_k(\cdot|x)-\pistar(\cdot|x)\Big)\Big( \Qht_k(x,\cdot)-\hatbonusQ_k(x,\cdot)\Big)\right]}_{\regterm}. 
\end{align*}

Now we combine the bounds in \pref{lem: linear MDP bias term}, \pref{lem: linear MDP bias term1}, and \pref{lem: linear MDP regret term} (included after this proof). Suppose that the conditions on the parameters specified in \pref{lem: linear MDP regret term} hold. We get  
\begin{align*}
    &\E[\bias + \biastwo + \biasthree + \biasfour + \regterm] \\
    &= \otil\left(\frac{H}{\eta} + \frac{\eta\epsilon H^4T}{\len} + \frac{\eta H^4}{\gamma^2 T^2\len} + \frac{\gamma dH^3T}{\beta\len} + \frac{\epsilon H^3T}{\len} \right)  \\
    &\qquad \qquad +  \left(\frac{\beta}{2} + 2\eta H^3\right)\sum_{k}\sum_h\E_{X_h \ni x\sim\pistar} \left[\one[x\in \calK]\sum_a (\pistar(a|x) + \pi_k(a|x))\|\phi(x,a)\|_{\hatcov_{k,h}}^2\right] \\  
    &\qquad \qquad + \frac{1}{H}\sum_{k}\sum_h\E_{X_h \ni x\sim\pistar} \left[\sum_a \pi_k(a|x)\bonusQ_k(x,a)\right] \\
    &\leq \otil\left(\frac{H}{\eta} + \frac{\eta\epsilon H^4T}{\len} + \frac{\eta H^4}{\gamma^2 T^2\len} + \frac{\gamma dH^3T}{\beta\len} + \frac{\epsilon H^3T}{\len} \right)  \\
    &\qquad \qquad +  \sum_{k}V^\pistar(x_0; b_k)
    + \frac{1}{H}\sum_{k}\sum_h\E_{X_h \ni x\sim\pistar} \left[\sum_a \pi_k(a|x)\bonusQ_k(x,a)\right]
\end{align*}
where the last inequality is because $\frac{\beta}{2} + 2\eta H^3 \leq \beta$ (implied by $\frac{\eta}{\beta}\leq \frac{1}{20H^4}$, a condition specified in \pref{lem: linear MDP regret term}). 

Combining two cases and applying \pref{lem: expected version}, we thus have 
\begin{align*}
    &\E\left[\sum_{k=1}^{(T-T_0)/\len} V^{\pi_k}(x_0; \barell_k)\right] - \sum_{k=1}^{(T-T_0)/\len} V^{\pistar}(x_0; \barell_k) \\
    &\leq  \otil\left(\frac{H}{\eta} + \frac{\eta\epsilon H^4T}{W} + \frac{\eta  H^4}{\gamma^2 T^2W} + \frac{\gamma dH^3 T}{\beta \len} +  \frac{\epsilon H^3T}{W} + \beta\E\left[\sum_{k=1}^{(T-T_0)/\len} \sum_{h}\E_{(x_h, a) \sim \pi_k}\left[ \|\phi(x_h,a)\|_{\hatcov_{k,h}}^2 \right]\right] + \frac{dH^3T}{\alpha\len}\right) \\
    &= \otil\left(\frac{H}{\eta} + \frac{\eta\epsilon H^4T}{W} + \frac{\eta  H^4}{\gamma^2 T^2W} + \frac{\gamma dH^3 T}{\beta \len} +  \frac{\epsilon H^3T}{W} + \frac{\beta dH T}{\len} + \frac{dH^3T}{\alpha\len}\right). \tag{by similar calculation as \pref{eq:LB_stability}}
\end{align*}

Finally, to get the overall regret, it remains to multiply the bound above by $\len$,
add the trivial bound $HT_0 =2HM_0 N_0 = \order\left( \frac{\explore^8d^{10}H^{11}}{\beta^8} \right)$ for the initial $T_0$ rounds,
and consider the exploration probability $\explore$, which leads to
\begin{align*}
    \E[\Reg]& = \otil\left(\frac{H\len}{\eta} + \eta\epsilon H^4T + \frac{\eta H^4 }{\gamma^2 T^2} + \frac{\gamma dH^3 T}{\beta} + \epsilon H^3T + \beta dH T + \frac{dH^3T }{\alpha} + \frac{d^{10}H^{11}\explore^8}{\beta^8} + \explore HT \right) \\
    &=\otil\left(\frac{H}{\eta\epsilon^2\gamma^3} + \frac{\eta H^4}{\gamma^2 T^2} + \frac{\gamma dH^3 T}{\beta} + \epsilon H^3T +  \beta dH T + \frac{dH^3T \beta}{\explore} + \frac{d^{10}H^{11}\explore^8}{\beta^8} +\explore HT\right)
\end{align*}
where we use the specified value of $M$, $N$, $\len=2MN$, $\alpha$, and that $\eta H\leq 1$ (so that the second term $\eta\epsilon H^4T$ is absorbed by the fifth term $\epsilon H^3T$). 

Considering the constraints in \pref{lem: linear MDP regret term} and \pref{lem: bonus bounded by 1}, we choose $\gamma = \max\left\{16\eta H^4, \frac{4\beta^2}{\explore}\right\}$. This gives the following simplified regret  
\begin{align*}
    &\otil\left(\frac{1}{ \epsilon^2 \eta^4  H^{11}} + \frac{1}{\eta H^4 T^2} + \frac{\eta dH^7 T}{\beta} + \epsilon H^3 T + \beta dHT +  \frac{dH^3T \beta}{\explore} + \frac{d^{10}H^{11}\explore^8}{\beta^8} +\explore HT\right).
\end{align*}
Choosing $\explore$ optimally, and supposing $\eta \geq \frac{1}{T}$, the above is simplified to 
\begin{align*}
    &\otil\left(\frac{1}{ \epsilon^2 \eta^4  H^{11}} + \frac{\eta dH^7 T}{\beta} + \epsilon H^3 T + \beta dHT + H^2\sqrt{d\beta}T +  d^2H^{\nicefrac{35}{9}}T^{\nicefrac{8}{9}} \right) \\
    &=\otil\left(\frac{1}{ \epsilon^2 \eta^4  H^{11}} + \frac{\eta dH^7 T}{\beta} + \epsilon H^3 T + H^2\sqrt{d\beta} T + d^2H^{\nicefrac{35}{9}}T^{\nicefrac{8}{9}} \right)  \tag{choosing $\beta\leq \frac{H^2}{d}$}.
\end{align*}
Picking optimal parameters in the last expression, we get $\otil\left( d^2H^4 T^{\nicefrac{14}{15}} \right)$. 
\end{proof}

\begin{lemma}
\label{lem: linear MDP bias term}
     \begin{align*}
         &\E[\bias + \biastwo] \\
         &\leq \frac{\beta}{4} \E\left[  \sum_{k=1}^{(T-T_0)/\len}\sum_h\E_{X_h \ni x\sim\pistar}\left[\one[x\in \calK]\sum_{a} \Big(\pistar(a|x) + \pi_k(a|x)\Big) \norm{\phi(x,a)}_{\hatcov_{k,h}}^2\right]\right]  + \order\left(\frac{\gamma dH^3T}{\beta\len} + \frac{\epsilon H^3T}{\len}\right). 
     \end{align*}
\end{lemma} 
\begin{proof}   
    The proof of this lemma is similar to that of \pref{lem: linear Q bias term}, except that we replace $T$ by $(T-T_0)/\len$, and consider the averaged loss $\barell_k$ in an epoch instead of the single episode loss $\ell_t$: 
    \begin{align}
        &\E_k\left[\barQ^{\pi_k}_k(x,a) - \Qht_k(x,a)\right]  \nonumber  \\
        &=\phi(x,a)^\top \left(\bartheta_{k,h}^{\pi_k} - \E_k\left[\hattheta_{k,h}\right]\right) \nonumber \\ 
        &= \phi(x,a)^\top \left(\bartheta_{k,h}^{\pi_k} - \E_k\left[\hatcov_{k,h}\right]\E_k\left[ \frac{1}{|S_k'|}\sum_{t\in S_k'} ((1-Y_t) + Y_t H\ind[h=h_t^*]) \phi(x_{t,h},a_{t,h})L_{t,h} \right]\right) \tag{$S_k'$ is the $S'$ in \pref{alg:linearMDP} within epoch $k$} \\
        &= \phi(x,a)^\top \left(\bartheta_{k,h}^{\pi_k} - \left( \gamma I + \cov_{k,h}^{\mix}\right)^{-1}\E_k\left[ \frac{1}{|S_k'|}\sum_{t\in S_k'} ((1-Y_t) + Y_t H\ind[h=h_t^*]) \phi(x_{t,h},a_{t,h})L_{t,h} \right]\right) + \order(\epsilon H^2) \tag{by \pref{lem: GR lemma contingency} and that $\|\phi(x,a)\|\leq 1$ for all $x,a$ and $L_{t,h}\leq H$; $\cov_{k,h}^{\mix}$ is defined in \pref{eq:mix_covariance}} \\
        &=  \phi(x,a)^\top \left(\bartheta_{k,h}^{\pi_k} - \left( \gamma I + \cov_{k,h}^{\mix}\right)^{-1}\E_k\left[  \frac{1}{|S_k'|}\sum_{t\in S_k'}\cov_{k,h}^{\mix} \theta^{\pi_k}_{t,h} \right]\right)  + \order(\epsilon H^2)  \nonumber \\
        &=  \phi(x,a)^\top \left(\bartheta_{k,h}^{\pi_k} - \left( \gamma I + \cov_{k,h}^{\mix}\right)^{-1}\E_k\left[  \frac{1}{\len}\sum_{t=(k-1)\len+1}^{k\len}\cov_{k,h}^{\mix} \theta^{\pi_k}_{t,h} \right]\right)  + \order(\epsilon H^2)  \tag{$S_k'$ is randomly chosen from epoch $k$}  \\  
        &= \phi(x,a)^\top \left(\bartheta_{k,h}^{\pi_k} - \left( \gamma I + \cov_{k,h}^{\mix}\right)^{-1}\cov_{k,h}^{\mix} \bartheta_{k,h}^{\pi_k}\right) + \order(\epsilon H^2) \nonumber \\
        &= \gamma \phi(x,a)^\top \left(\gamma I + \cov^{\mix}_{k,h}\right)^{-1} \bartheta^{\pi_k}_{k,h} + \order\left(\epsilon H^2\right) \nonumber  \\
        &\leq \frac{\beta}{4}  \|\phi(x,a)\|_{(\gamma I + \cov_{t,h}^{\mix})^{-1}}^2  + \frac{\gamma^2}{\beta}  \norm{\bartheta_{k,h}^{\pi_k}}_{(\gamma I + \cov_{k,h}^{\mix})^{-1}}^2 + \order(\epsilon H^2) \nonumber \tag{AM-GM inequality}  \\
        &\leq \frac{\beta}{4} \E_k\left[  \norm{\phi(x,a)}_{\hatcov_{k,h}}^2\right] + \frac{\gamma dH^2}{\beta} + \order\left(\epsilon H^2\right).   \nonumber
        \label{eq: long proff 4} 
    \end{align}
    The same bound also holds for $\E_k\left[\Qht_k(x,a) - \barQ^{\pi_k}_k(x,a) \right]$ by the same reasoning. 
    Taking expectation over $x$, summing over $k, h$ and $a$ (with weights $\pi_k(a|x)$ and $\pistar(a|x)$ respectively) finishes the proof. 
\end{proof}

\begin{lemma}
\label{lem: linear MDP bias term1}
     \begin{align*}
         &\E[\biasthree + \biasfour]\\
         &\leq \frac{\beta}{4} \E\left[  \sum_{k=1}^{(T-T_0)/\len}\sum_h\E_{X_h \ni x\sim\pistar}\left[\one[x\in \calK]\sum_{a} \Big(\pistar(a|x) + \pi_k(a|x)\Big) \norm{\phi(x,a)}_{\hatcov_{k,h}}^2\right]\right] + \order\left(\frac{\gamma dH^3 T}{\beta\len} + \frac{\epsilon  H^3 T}{ \len}\right).
     \end{align*}
\end{lemma} 
\begin{proof}
    The proof is almost identical to that of the previous lemma. The only difference is that $L_{t,h}$ is replaced by $D_{t,h}$ and  $\bartheta^{\pi_k}_{t,h}$ is replaced by $\Lambda^{\pi_k}_{k,h}$ (recall the definition of $\Lambda_{k,h}^{\pi_k}$ in \pref{sec: linear MDP}). Note that $\bonus_t(x,a)\in [0,1]$ (\pref{lem: bonus bounded by 1}), so $D_{t,h}\in [0, He]$, which is also the same order for $L_{t,h}$. Therefore, we get the same bound as in the previous lemma. 
\end{proof}

\begin{lemma}\label{lem: linear MDP regret term}
    Let $\frac{\eta }{\gamma}\leq \frac{1}{16H^4}$  and  $\frac{\eta}{\beta}\leq \frac{1}{40H^4}$. Then
    \begin{align*}
        &\E[\regterm]=\otil\left(\frac{H}{\eta} + \frac{\eta\epsilon H^4 T}{\len} + \frac{\eta H^4}{\gamma^2 T^2\len}\right)  \\
        &  + 2\eta H^3\E\left[\sum_{k,h}\E_{X_h \ni  x\sim\pistar}\left[\one[x\in \calK]\sum_a\pi_k(x,a)\|\phi(x,a)\|_{\hatcov_{k,h}}^2 \right]\right]  + \frac{1}{H} \E\left[\sum_{k,h}\E_{X_h \ni  x\sim\pistar}\left[\sum_a\pi_k(x,a)\bonusQ_k(x,a) \right] \right]. 
    \end{align*}
\end{lemma}
\begin{proof}
    We first check the condition for \pref{lem: exponential weight lemma}: $\eta \left|\Qht_k(x,a) - \hatbonusQ_t(x,a)\right|\leq 1$. 
    In our case, 
    \begin{align*}
        \eta \left|\Qht_k(x,a)\right| 
        &= \eta \left| \phi(x,a)^\top \hatcov_{k,h}\left(\frac{1}{|S'|}\sum_{t\in S'} ((1-Y_t) + Y_t H\ind[h=h_t^*]) \phi(x_{t,h},a_{t,h})L_{t,h} \right) \right|  \\
        &\leq \eta \times \|\hatcov_{k,h}\|_{\text{op}}\times H \times \sup_{t\in S'} L_{t,h} \\
        &\leq \eta \times \frac{1}{\gamma} \times H^2 \tag{by \pref{lem: GR lemma contingency}} \\ 
        &\leq \frac{1}{2} \tag{by the condition specified in the lemma}
    \end{align*}
    and 
    \begin{align*}
        \eta\left|\hatbonusQ_k(x,a)\right| &\leq \eta \left|\bonus_k(x,a)\right| + \eta \left| \phi(x,a)^\top  \hatcov_{k,h}\left(\frac{1}{|S'|}\sum_{t\in S'} ((1-Y_t) + Y_t H\ind[h=h_t^*]) \phi(x_{t,h},a_{t,h})D_{t,h} \right) \right| \\
        &\leq \eta + \eta \times \|\hatcov_{k,h}\|_{\text{op}} \times  H \times \sup_{t\in S'} D_{t,h} \\
        &\leq \eta + \eta \times \|\hatcov_{k,h}\|_{\text{op}} \times  H \times (H-1)\left(1+\frac{1}{H}\right)^H \tag{\pref{lem: bonus bounded by 1}}\\
        &\leq \eta + \frac{3\eta H^2}{\gamma} \leq \frac{4\eta H^2}{\gamma} \\
        &\leq \frac{1}{2H}.  \tag{by the condition specified in the lemma}
    \end{align*}
    
    Now we derive an upper bound for $\E_k\left[\Qht_k(x,a)^2\right]$: 
    \begin{align}
        &\E_k\left[\Qht_k(x,a)^2\right]  \nonumber \\
        &\leq \E_k\left[\frac{1}{|S_k'|}\sum_{t\in S_k'}  H^2\phi(x,a)^\top \hatcov_{k,h}\Big(((1-Y_t)+Y_t H \ind[h=h_t^*])^2  \phi(x_{t,h}, a_{t,h})\phi(x_{t,h}, a_{t,h})^\top\Big) \hatcov_{k,h} \phi(x,a)\right]  \tag{$*$}\\ 
        &= \E_k\left[H^2\phi(x,a)^\top \hatcov_{k,h}\Big((1-\explore)  \cov_{k,h} + \explore H \cov_{h}^{\textcov} \Big) \hatcov_{k,h} \phi(x,a)\right]  \nonumber \\
        &\leq H^3\E_k\left[\phi(x,a)^\top \hatcov_{k,h}\cov_{k,h}^{\mix} \hatcov_{k,h} \phi(x,a)\right] \nonumber \\ 
        &\leq H^3\E_k\left[\phi(x,a)^\top \hatcov_{k,h}\cov_{k,h}^{\mix} (\gamma I+\cov_{k,h}^{\mix})^{-1}\phi(x,a)\right] + \otil\left(\epsilon H^3 + \frac{H^3}{\gamma^2 T^3}\right)  \tag{\pref{lem: GR lemma contingency}} \\ 
        &\leq H^3\phi(x,a)^\top (\gamma I+\cov_{k,h}^{\mix})^{-1}\cov_{k,h}^{\mix} (\gamma I+\cov_{k,h}^{\mix})^{-1}\phi(x,a) + \otil\left(\epsilon H^3 + \frac{H^3}{\gamma^2 T^3}\right)  \tag{\pref{lem: GR lemma contingency}} \\
        &\leq H^3\phi(x,a)^\top (\gamma I+\cov_{k,h}^{\mix})^{-1}\phi(x,a) + \otil\left(\epsilon H^3 + \frac{H^3}{\gamma^2 T^3}\right)  \nonumber \\ 
        &= H^3\E_k\left[\|\phi(x,a)\|_{\hatcov_{k,h}}^2 \right] + \otil\left(\epsilon H^3 + \frac{H^3}{\gamma^2 T^3}\right),   \label{eq: follow 2}
    \end{align}
    where in $(*)$ we use
    $\left(\frac{1}{|S_k'|}\sum_{t\in S_k'}v_t\right)^2\leq \frac{1}{|S_k'|}\sum_{t\in S_k'}v_t^2$ with $v_t = \phi(x,a)^\top \hatcov_{k,h}\left((1-Y_t) + Y_t H \ind[h=h_t^*]\right)\phi(x_{t,h}, a_{t,h})L_{t,h}$.

    Next, we bound $\E_t\left[\hatbonusQ_t(x,a)^2\right]$: 
    \begin{align*}
        &\E_k\left[\hatbonusQ_k(x,a)^2\right] \\
        &\leq 2\E_k\left[\bonus_k(x,a)^2\right] + 2\E_k\left[(\phi(x,a)^\top \bonusVec_{k,h})^2\right]\\
        &\leq 2\E_k[\bonus_k(x,a)]  + 18H^3 \E_k\left[\|\phi(x,a)\|_{\hatcov_{k,h}}^2 \right] + \otil\left(\epsilon H^3 + \frac{H^3}{\gamma^2 T^3}\right)\\
        &\leq  \frac{20H^3}{\beta}  \bonus_k(x,a) + \otil\left(\epsilon H^3 + \frac{H^3}{\gamma^2 T^3}\right),
    \end{align*}
    where in the second inequality we bound $\E_k\left[(\phi(x,a)^\top \bonusVec_{k,h})^2\right]$ similarly as we bound $\E_k\left[\Qht_k(x,a)^2\right]$ in \pref{eq: follow 2}, except that we replace the upper bound $H$ for $L_{t,h}$ by the upper bound for $D_{t,h}$: $H\left(1+\frac{1}{H}\right)^H  \sup_{t,x,a}\bonus_t(x,a) \leq 3H$ (since $\bonus_t(x,a)\leq 1$ by \pref{lem: bonus bounded by 1}). 
    
    Thus, by \pref{lem: exponential weight lemma}, we have 
    \begin{align*}
        &\E[\regterm] \\
        &\leq \otil\left(\frac{H}{\eta}\right) + 2\eta\sum_{k,h}\E_{X_h \ni x\sim \pistar}\left[\one[x \in \calK] \sum_{a} \pi_k(a|x) (\Qht_k(x,a)^2 + \hatbonusQ_k(x,a)^2)\right]  \\
        &\leq \otil\left(\frac{H}{\eta} + \frac{\eta\epsilon H^4T}{\len} + \frac{\eta H^4}{\gamma^2 T^2\len}\right) \\
        & + 2\eta H^3\E\left[\sum_{k, h}\E_{X_h \ni x\sim \pistar}\left[\one[x \in \calK] \sum_{a} \pi_k(x,a)\|\phi(x,a)\|_{\hatcov_{k,h}}^2 \right]\right] + \frac{40\eta H^3}{\beta} \E\left[\sum_{k,h}\E_{X_h \ni x\sim \pistar}\left[ \sum_{a}\pi_k(a|x)\bonus_k(x,a) \right]\right] \\
        &\leq \otil\left(\frac{H}{\eta} + \frac{\eta\epsilon H^4T}{\len} + \frac{\eta H^4}{\gamma^2 T^2\len}\right) \\
        &+ 2\eta H^3\E\left[\sum_{k, h}\E_{X_h \ni x\sim \pistar}\left[\one[x \in \calK] \sum_{a} \pi_k(x,a)\|\phi(x,a)\|_{\hatcov_{k,h}}^2 \right]\right] + \frac{1}{H} \E\left[\sum_{k,h}\E_{X_h \ni x\sim \pistar}\left[ \sum_{a}\pi_k(a|x)\Bonus_k(x,a) \right]\right] 
    \end{align*}
    where in the last inequality we use the conditions specified in the lemma and that $\bonusQ_k(x,a)\geq \bonus_k(x,a)$. 
\end{proof}

\end{document}